\documentclass[journal, 10pt]{IEEEtran}

\usepackage{cite}
\usepackage{amsmath,amssymb,amsfonts}
\usepackage{algorithmic}
\usepackage{graphicx}
\usepackage{textcomp}
\usepackage{xcolor}
\usepackage{subfig}
\usepackage{amsthm}
\usepackage{mathrsfs}

\usepackage{pdfpages}
\usepackage{colortbl}
\definecolor{mygray}{gray}{0.85} 
\usepackage{array}
\usepackage{caption}
\usepackage{setspace}
\usepackage{flushend}  %双栏布局最后一页底部对齐的处理
\usepackage[linesnumbered,boxed,ruled,commentsnumbered]{algorithm2e}%%算法包，注意设置所需可选项 algorithm 
\usepackage{float}
\IncMargin{1em}

\usepackage[ruled,linesnumbered]{algorithm2e}

\usepackage{amsmath,bm}
\usepackage{amsthm}

\begin{document}
\begin{spacing}{1.0}
\title{Joint Optimization of Cooperation Efficiency and Communication Covertness for Target Detection with AUVs
%Covert and Efficient Task Allocation for Underwater Multi-Agent Systems using Hierarchical Reinforcement Learning
%Multi-Agent Cooperation Design for Internet of Underwater Things: A Hierarchical Reinforcement Learning Approach
%Cooperative Target Detection with AUVs: \\ A Hierarchical Reinforcement Learning Approach
%Joint Resource Allocation and Task Offloading for RICS-Assisted MEC Systems: A Driving Safety-Enabled MADRL Approach
}
\author{ 
Xueyao~Zhang,~\IEEEmembership{Student Member,~IEEE,} 
Bo~Yang,~\IEEEmembership{Senior Member,~IEEE,} Zhiwen Yu,~\IEEEmembership{Senior Member,~IEEE,}  Xuelin Cao,~\IEEEmembership{Senior Member,~IEEE,}  Wei Xiang, Bin Guo,~\IEEEmembership{Senior Member,~IEEE,} Liang Wang,~\IEEEmembership{Member,~IEEE,}   Billy Pik Lik Lau,~\IEEEmembership{Senior Member,~IEEE,} George C. Alexandropoulos, \IEEEmembership{Senior Member,~IEEE,} Jun Luo,\\ \IEEEmembership{Fellow,~IEEE}, M\'erouane Debbah,~\IEEEmembership{Fellow,~IEEE},  Zhu Han,~\IEEEmembership{Fellow,~IEEE}, 
%Ian F. Akyildiz,~\IEEEmembership{Life Fellow,~IEEE}, 
and Chau~Yuen,~\IEEEmembership{Fellow,~IEEE}
 \thanks{X. Zhang, B. Yang, L. Wang, and B. Guo are with the School of Computer Science, Northwestern Polytechnical University, Xi'an, Shaanxi, 710129, China (email: 2024263006@mail.nwpu.edu.cn, yang$\_$bo, liangwang, guob@nwpu.edu.cn). 

Z. Yu is with the School of Computer Science, Northwestern Polytechnical University, Xi'an, Shaanxi, 710129, China, and Harbin Engineering University, Harbin, Heilongjiang, 150001, China (email: zhiwenyu@nwpu.edu.cn).

 X. Cao is with the School of Cyber Engineering, Xidian University, Xi'an, Shaanxi, 710071, China (email: caoxuelin@xidian.edu.cn). 

 W. Xiang is with Yichang Testing Technology Research Institute, Yichang,  Hubei, 443000, China, and School of Computer Science, Northwestern Polytechnical University, Xi'an, Shaanxi, 710129, China (email: 46635991@qq.com)

 G. C. Alexandropoulos is with the Department of Informatics and Telecommunications, National and Kapodistrian University of Athens, 16122 Athens, Greece (email: alexandg@di.uoa.gr). 

J. Luo is with the College of Computing and Data Science, Nanyang
Technological University, Singapore 639798 (email: junluo@ntu.edu.sg).
%M. Li is with the Department of Computer Science and Engineering, The Hong Kong University of Science and Technology, Hong Kong (email: lim@cse.ust.hk).
% Y. Zhang is with the Department of Informatics, University of Oslo, 0316 Oslo, Norway (email: anzhang@ieee.org).  

M. Debbah is with  KU 6G Research Center, Department of Computer and Information Engineering, Khalifa University, Abu Dhabi 127788, UAE (email: merouane.debbah@ku.ac.ae).

Z. Han is with the Department of Electrical and Computer Engineering at the University of Houston, Houston, TX 77004 USA, and also with the Department of Computer Science and Engineering, Kyung Hee University, Seoul, South Korea, 446-701(email: hanzhu22@gmail.com).
%I. F. Akyildiz is with the Center for Robotics and Wireless Communications in Challenging Environments, Electrical and Computer Engineering, University of Iceland IS102 Reykjavik, Iceland (email: ianaky@hi.is).
%Truva Inc., Alpharetta, GA 30022, USA, and with Odine Research Labs, Istanbul, Turkiye, (email: ian.akyildiz@odine.com, ian@truvainc.com)

B. P. L. Lau and C. Yuen are with the School of Electrical and Electronics Engineering, Nanyang Technological University, Singapore (email: billy.laupl, chau.yuen@ntu.edu.sg).}
}

% The paper headers
%\markboth{IEEE Wireless Communications Letters,~Vol.~XX, No.~XX, XXX~2020}
%{}
% The paper headers
%\markboth{Journal of \LaTeX\ Class Files,~Vol.~14, No.~8, August~2015}%
%{Shell \MakeLowercase{\textit{et al.}}: Bare Demo of IEEEtran.cls for IEEE Journals}

\maketitle

\begin{abstract}
This paper investigates underwater cooperative target detection using autonomous underwater vehicles (AUVs), with a focus on the critical trade-off between cooperation efficiency and communication covertness. To tackle this challenge, we first formulate a joint trajectory and power control optimization problem, and then present an innovative hierarchical action management framework to solve it. According to the hierarchical formulation, at the macro level, the master AUV models the agent selection process as a Markov decision process and deploys the proximal policy optimization algorithm for strategic task allocation. At the micro level, each selected agent's decentralized decision-making is modeled as a partially observable Markov decision process, and a multi-agent proximal policy optimization algorithm is used to dynamically adjust its trajectory and transmission power based on its local observations. Under the centralized training and decentralized execution paradigm, our target detection framework enables adaptive covert cooperation while satisfying both energy and mobility constraints. By comprehensively modeling the considered system, the involved signals and tasks, as well as energy consumption, theoretical insights and practical solutions for the efficient and secure operation of multiple AUVs are provided, offering significant implications for the execution of underwater covert communication tasks.
\end{abstract}


\IEEEpeerreviewmaketitle

\section{Introduction}
\IEEEPARstart{C}{urrently}, the rapid breakthroughs in artificial intelligence (AI) technology and the continuous evolution of communication network technology are profoundly driving various industries towards a new stage of extensive intelligence and large-scale networking. Against the backdrop of this era of technological integration, the demand for efficient exploration, precise perception, and intelligent operation in the vast and complex ocean space is becoming increasingly urgent, driving underwater agents (such as autonomous underwater vehicles (AUVs) and underwater sensors) to gradually evolve from the traditional isolated operation mode to a new operation paradigm with a distributed structure and cluster collaboration capabilities \cite{Underwater}. As a key approach towards this demand, the Internet of Underwater Things (IoUT) has emerged as a novel paradigm, aiming to build a collaborative network composed of heterogeneous underwater nodes, enabling seamless integration of environmental sensing, communications, and intelligent control \cite{IoU01,IoU02}.
By constructing the IoUT, not only can the operation range be significantly expanded and the efficiency of information acquisition and processing be improved, but also the adaptability and robustness of the system in the unknown and dynamic underwater environment can be enhanced, thereby relying on collaborative mechanisms among agents.

\subsection{Background and Challenges}
AUVs constitute a highly autonomous and intelligent platform for underwater transportation and operations. They offer significant advantages due to their tether-free operation and the ability to carry a diverse array of sensors for prolonged and extensive exploratory missions \cite{AUV}. As a result, AUVs have been widely deployed in fields such as environmental monitoring and resource exploration. However, when executing long-term or large-scale tasks in complex and dynamic underwater environments, individual AUVs often face limitations due to their restricted onboard energy and limited sensing range, making it difficult to meet task requirements effectively. Consequently, establishing a distributed collaborative network composed of multiple agents has emerged as an effective approach to enhance the efficiency of underwater task execution and the overall capabilities of the system.

However, the challenges faced during underwater operations are significant and cannot be overlooked. The underwater acoustic channel, which serves as the primary communication medium, is frequently plagued by issues, such as high propagation delay, limited bandwidth, pronounced multipath effects, and susceptibility to interference \cite{underwater_challenge}, all of which severely restrict the efficiency and reliability of information exchange between agents. Additionally, difficulties in precise underwater positioning and navigation, coupled with the limited energy and computational resources of the agents themselves, impose stringent requirements on the design and implementation of multi-agent collaborative mechanisms.
Moreover, in certain underwater application scenarios, such as military reconnaissance, environmental monitoring in sensitive areas, and specialized underwater operations, covertness becomes a paramount requirement that must take precedence over collaboration. Traditional open communication and explicit cooperative behaviors can easily reveal the presence of agents, potentially leading to mission failure and systemic risks.

Due to the critical importance of covertness, the current deployment and application of underwater agents largely favor an ``each for themselves" independent operational mode, aimed at minimizing the risk of detection by non-cooperative entities \cite{joint}. While progress has been made in enhancing the covert navigation capabilities of individual AUVs, designing low-probability intercept (LPI) communication waveforms, and exploring biomimetic covertness communications often comes at the cost of sacrificing the potential for multi-agent collaboration. 
Under the imperative of maintaining covertness, a significant challenge arises: how to break down the ``information silos" between agents and enable effective sharing of situational awareness, task coordination, and resource optimization. Achieving this, while ensuring system security and fully leveraging the advantages of collaborative strategies, has become a critical bottleneck that urgently needs to be addressed in the field of underwater intelligent systems research.

\subsection{Motivation and Contributions}
When addressing the complex decision-making challenges of achieving covert collaboration in the aforementioned dynamic and resource-constrained underwater environments, traditional optimization methods \cite{trad} face significant limitations. For instance, conventional centralized approaches based on convex optimization theory can find global optimal solutions under specific conditions. However, they typically rely on precise mathematical descriptions of the environmental model and the complete availability of global information. 
This reliance results in these methods providing only static solutions tailored to specific scenarios or environmental parameters. Once disturbances occur, the model must be reconstructed and the optimization solution needs to be re-computed. Such sensitivity to environmental changes leads to insufficient generalization in dynamic scenarios, making it difficult to meet the real-time and fast-changing requirements of underwater operations, thereby limiting the practical application potential of these systems in complex underwater settings.

In recent years, multi-agent deep reinforcement learning (MADRL) has emerged as an advanced framework that integrates distributed agent modeling, autonomous decision-making, and environmental interaction mechanisms, offering a fresh perspective and valuable tools for solving multi-agent collaborative control problems. MADRL enables each agent to autonomously learn optimal behavioral strategies through interactions with the environment, fostering collaboration among agents to achieve common objectives. Furthermore, MADRL exhibits strong policy transfer capabilities and dynamic adaptability, allowing it to continuously optimize decision-making behaviors in uncertain and frequently disturbed underwater environments. The significant advantages of MADRL in distributed collaboration, adaptation to dynamic environments, and learning-to-optimize complex behaviors provide a promising research framework for addressing the covert collaboration challenges faced by underwater multi-agent systems.

To our knowledge, this is an early study addressing the multi-AUVs collaboration problem while ensuring underwater covertness. To achieve this, we consider the high dynamics of the underwater environment, the communication limitations of the acoustic channel, and the distributed nature of agentic decision-making. Integrating agent task allocation and trajectory planning, we formulate a dual-scale cooperative optimization problem. Furthermore, we develop a hierarchical multi-agent proximal policy optimization (HMAPPO) framework tailored for covert underwater multi-agent collaboration. This framework aims to collaboratively optimize macro-level agent allocation and micro-level trajectory control through intelligent hierarchical decision-making and learning mechanisms, all while maintaining the overall system covertness to maximize the collaborative task-handling capabilities of the agents' cluster.

The remainder of the paper is organized as follows. Section~\ref{related-works} reviews related works on underwater acoustic communications, covert transmission, and multi-agent reinforcement learning, while Section~\ref{system} introduces the considered underwater multi-agent system model and presents the formulation of the covert cooperation problem. In Section~\ref{method}, the proposed HMAPPO design framework is detailed. Section~\ref{result} presents the paper's numerical investigations, including our design's convergence, scalability, and the performance trade-off with covertness. Finally, Section~\ref{conclusion} concludes the paper.

\section{Literature Review} \label{related-works}
%Underwater acoustic communications is the primary means of information exchange between wireless devices underwater. However, due to the underwater channel characteristics, such as high latency, significant multipath effects, and limited available bandwidth \cite{UWA}, there are substantial differences compared to terrestrial wireless communication, which pose significant challenges to maintaining efficient network performance. %To address these challenges, researchers have conducted extensive explorations in areas such as modulation and demodulation techniques, channel coding, equalization, and multiple access methods. For example, studies have focused on OFDM technology, multiple access techniques, and spread spectrum techniques.

\subsection{Underwater Covert Communications}
In recent years, research on underwater convert communications has become increasingly active. The goal is to achieve reliable information transmission between legitimate users without being detected by unauthorized parties (eavesdroppers or detectors). Currently, the main technologies for implementing underwater covert communication include: \textit{i}) \textit{Spread spectrum}: 
By spreading the signal energy over a wider frequency band, the power spectral density can be significantly reduced, making the signal indistinguishable from background noise, thereby yielding low probability of interception (LPI) and low probability of detection (LPD)~\cite{ss01,ss02}.
\textit{ii}) \textit{Power control}: 
This approach adaptively adjusts the transmit power to blend the signal into the noise at the eavesdropper's received signal. By strategically limiting the power level, the transmitter ensures a low signal-to-noise (SNR) at the eavesdropper, forcing a high detection error rate, thus, effectively hiding the transmission within channel uncertainty~\cite{power_control01} \cite{power_control02}.
\textit{iii} \textit{Waveform design}: \cite{covert} proposed to utilize the noise in the marine environment to design the communication signal into a waveform that is statistically similar to the background noise, or to use the background noise as a covert for transmission \cite{noise}.

Commonly used methods for underwater signal analysis include energy detection methods~\cite{detection}, which, however, exhibit significant limitations in low SNR conditions. A theoretically optimal detection scheme based on the likelihood ratio test (LRT) was presented in \cite{LRT}, which shown to be capable of improving detection performance under consistent statistical characteristics of information-bearing signals and noise. However, the implementation of LRT requires precise statistical properties, which often present challenges in practical applications. In~\cite{LPI}, the impact of LPI and LPD in covert communications was further explored. Finally, \cite{joint} derived the achievable covert communication rates between users under the constraint of Kullback-Leibler (KL) divergence.

\subsection{Multi-Agent Deep Reinforcement Learning Approaches}
MADRL aims to enable multiple agents to learn how to collaborate (or compete) in achieving common or individual goals through interactions with the environment and with each other. Existing MADRL algorithm frameworks include the independent learners framework \cite{IL}, where each agent treats other agents as part of the environment and learns its strategy independently; however, this approach faces challenges related to non-stationarity. 
Policy gradient methods constitute another commonly used framework, such as the multi-agent deep deterministic policy gradient (MADDPG) method \cite{DDPG}, which optimizes via deep learning a deterministic policy for each agent along with a centralized critic network. The Centralized Training with Decentralized Execution (CTDE) architecture presengted in~\cite{CTDE} is currently one of the most widely applied frameworks in the field. During the training phase, agents are allowed to access global state information to learn a centralized evaluation function; however, during the execution phase, agents rely solely on local observations for decision-making. This framework alleviates non-stationarity issues to some extent, facilitating implicit coordination among agents.

MADRL has demonstrated a broad range of application prospects in diverse fields, such as robotics and drone collaboration, including formation control \cite{MADRL01}, path planning \cite{MADRL02}, target search \cite{MADRL03}, and task allocation \cite{MADRL04}. These applications demonstrate the substantial potential of MADRL in managing dynamic interactions among multiple agents and in learning complex collaborative strategies. In \cite{MADRL1}, a hierarchical heterogeneous multi-agent framework for marine target searching, which integrates UAVs, surface vessels, and underwater vehicles, and enhances system coordination through adaptive task decomposition was presented. In the aerial network domain~\cite{MADRL2}, a hierarchical computing system composed of high-altitude platforms and unmanned aerial vehicles (UAVs), employing a MAPPO-based algorithm for efficient resource management and task offloading was proposed. A clustered multi-UAV system was deployed using MADRL to minimize the computation cost for IoT devices in~\cite{MADRL3}. In \cite{MADRL_GA}, a distributed state-action-reward-state-action algorithm deciding the optimal sequential decision-making policies of multiple UAVs with the goal of data traffic offloading from terrestrial base stations was proposed. In the context of sustainable underwater networks~\cite{efficiency}, the joint optimization of data throughput and energy harvesting was modeled as a Markov decision process (MDP), and reinforcement learning was deployed to determine the optimal trajectory of AUVs for data collection and energy transfer. A DDPG-based two-timescale scheme was developed in~\cite{MADRL5} to jointly optimize long-term service caching and short-term resource allocation in edge computing. This hierarchical paradigm was further extended to the electric bus charging scenario in~\cite{MADRL6}. A double actor-critic MAPPO algorithm, where the high-level agent manages charger allocation and the low-level agent controls charging power, was presented. 

Inspired by the latter recent studies, in this paper, we  design a two-timescale framework to maximize the collaboration efficiency of underwater AUVs for cooperative target detection under covert communication constraints.

\theoremstyle{Observation}
\newtheorem{observation}{\textit{Observation}}
\theoremstyle{Lemma} 
\newtheorem{lemma}{{\textit{Lemma}}}

\section{System Modeling and Problem Formulation} \label{system}
\subsection{System Model}
% \begin{figure*}[!t]
% \centering
% 	\captionsetup{font={small}} 
% \includegraphics[width=2.05\columnwidth]{scene.pdf}
% \caption{The considered underwater covert communication scenario for cooperative target detection is shown in (a), and the structure of the communication time slot is illustrated in (b).}
% \label{scene} 
% \end{figure*}
\begin{figure*}[htpb] 
    \centering
    \captionsetup{font={small}} 
    \subfloat[][]
    {
    \includegraphics[width=1.22\columnwidth]{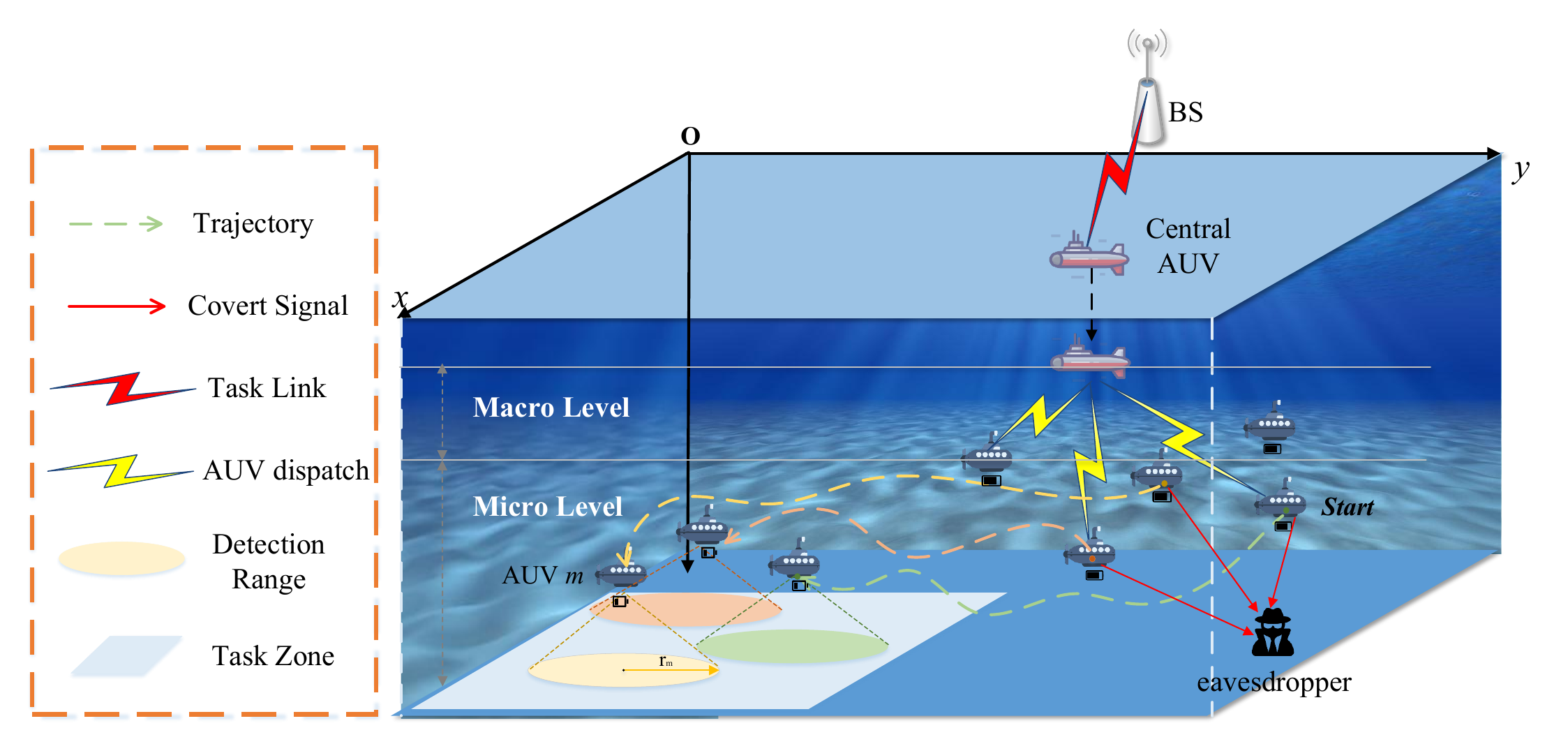}
    \label{scene}
    }
    \hfill
    \subfloat[][]
    {
    \includegraphics[width=0.73\columnwidth]{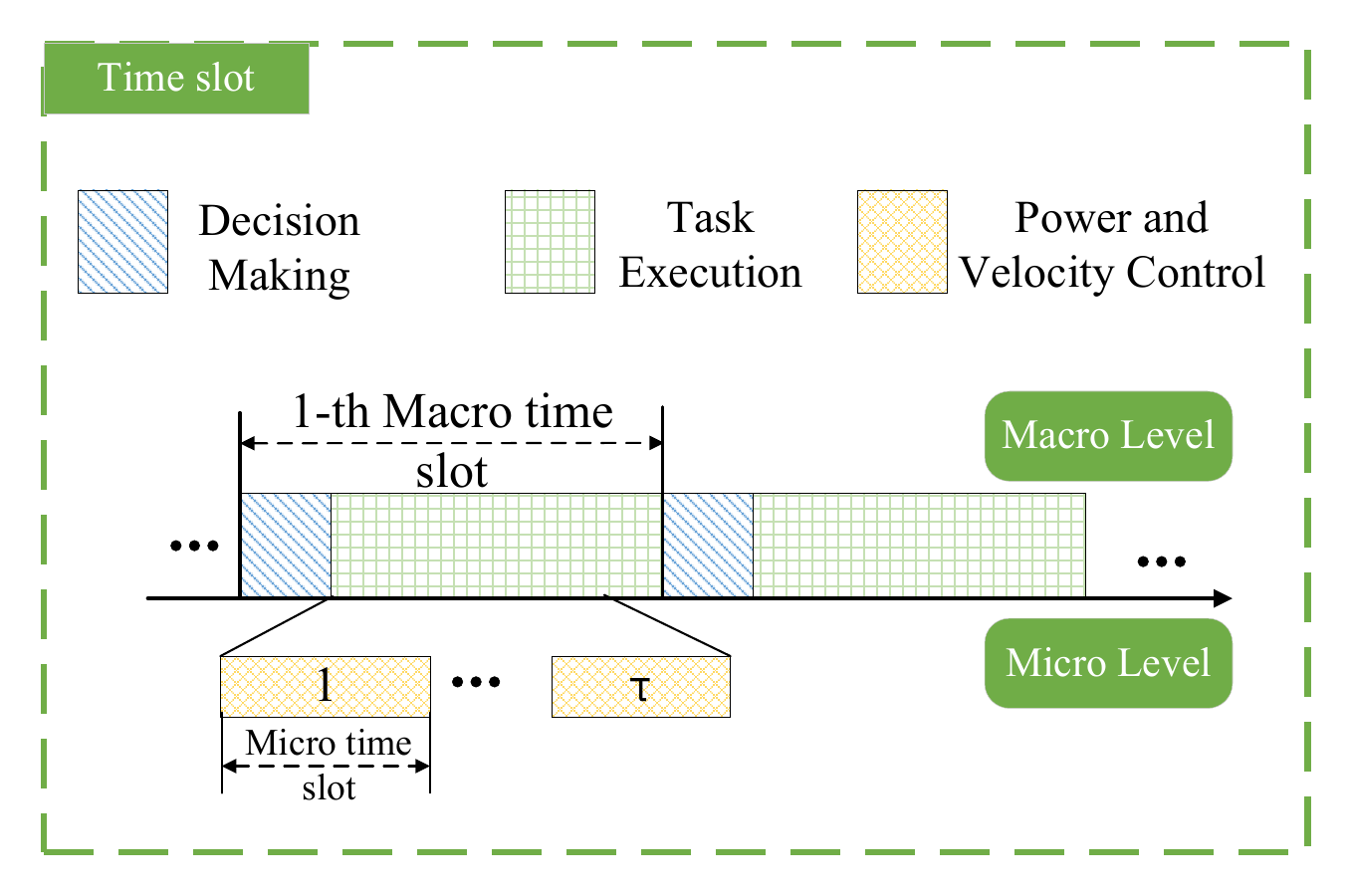}
    \label{timeslot}
    }
    \caption{The considered underwater covert communication scenario for cooperative target detection is shown in (a), and the structure of the communication time slot is illustrated in (b).}
\end{figure*}
We considered a multi-AUV communication system in a three-dimensional (3D) underwater setup, as illustrated in Fig.~\ref{scene}. The system consists of three core components: an AUV serving as the master control node, responsible for receiving and dispatching tasks; a group of $M$ cooperating AUVs, denoted as $\mathcal{M}=\left\{ 1,2,...,M \right\}$; and a passive eavesdropper $D$. The master AUV serves as the core command and control node of the system, with its mission cycle beginning when receiving task instructions from the shore-based base station (BS). This phase typically involves the master AUV surfing to ensure reliable communications. Once this node receives task instructions, which in the case study of this paper involves collaborative exploration of a specific underwater area, it processes and interprets the commands. After receiving the instructions successfully, the master AUV descends to the designated operational depth and begins executing its coordinated management function to accomplish the target detection task.

The considered task involves a long-term and complex decision-making problem consisting of variables spanning different timelines. To this end, we propose to decompose it into subproblems of varying granularity to optimize strategies at different levels. In particular, we define two decision-making hierarchies: one in macro time slots and the other on micro time slots, as depicted in Fig.~\ref{timeslot}.

\begin{itemize}
    \item {Macro time slots $t$}: A macro time slot corresponds to a distinct and complete subtask assigned to the AUV. A single episode consists of multiple consecutive macro time slots, which represent a series of interconnected tasks to be accomplished. At the beginning of each macro time slot, the central AUV makes a crucial macro decision: based on the current task characteristics and the status of the AUVs, it selects an appropriate subset of AUVs to execute the task. Once selected, the group of AUVs typically remains unchanged for the duration of the current macro time slot.
 \item {Micro time slots $\tau$}:  Each macro time slot $t$ can be further divided into multiple consecutive micro time slots, which constitute the actual operations for task execution. Within each micro time slot, the AUVs selected by the central AUV to participate in the current task must make fine-grained decisions about their actions.
 \end{itemize}
 
Note that here, the central AUV actively decides which AUVs to participate in this task based on either a preset strategy or a real-time assessment. This decision-making process is centrally controlled by the central AUV and through the variables $G=\{G_1,G_2,...,G_m\}$, where $G_m \in \{0,1\}$. Among them, $G_m=1$ indicates that the AUV $m$ is selected by the AUV to participate in the current task. The AUV will establish a communication connection with it and issue task instructions. If $G_m=0$, it indicates that the AUV $m$ does not participate in the task within this macro time slot.
Furthermore, the AUVs adjust their 3-D thrust velocity vectors based on the assigned target task location and environmental conditions to achieve effective path planning. At the same time, they must select the appropriate transmission power level in accordance with covertness constraints and communication requirements. Continuous decision-making during micro time slots is critical for the AUVs' navigation in complex environments, efficient task execution, and maintaining stealth from potential eavesdroppers. The duration of a micro time slot defines the minimal time unit for motion control and power adjustment of the AUVs within the system. 
It is assumed that within a single micro time slot, key environmental variables affecting AUV decisions and system interactions are considered quasi-static; however, these variables can evolve independently across successive micro time slots according to their own dynamics or stochastic processes, thereby reflecting the non-stationarity of the operational environment.

In the considered network scenario, $M$ AUVs are randomly distributed in the underwater 3D space, and we denote their position as: $\overset{\to }{\mathop{{{r}_{m}}}}\,\left[ t,\tau  \right]=\left( {{x}_{m}}\left[ t,\tau  \right],{{y}_{m}}\left[ t,\tau  \right],{{z}_{m}}\left[ t,\tau  \right] \right)$. Besides, AUV is located at: $\overset{\to }{\mathop{{{r}_{AUV}}}}\,\left[ t \right]=\left( {{x}_{AUV}}\left[ t \right],{{y}_{AUV}}\left[ t \right],{{z}_{AUV}}\left[ t \right] \right)$. And the eavesdropper is located in a fixed position: $\overset{\to }{\mathop{{{r}_{d}}}}\,=\left( {{x}_{d}},{{y}_{d}},{{z}_{d}} \right).$

The distances from the AUV to the central AUV and between the AUVs, respectively,
\begin{equation}
    {{d}_{mcAUV}}\left[ t,\tau  \right]={{\left\| {{\overset{\to }{\mathop{r}}\,}_{m}}\left[ t,\tau  \right]-{{\overset{\to }{\mathop{r}}\,}_{cAUV}}\left[ t,\tau  \right] \right\|}_{2}},
\end{equation}

\begin{equation}
    {{d}_{ij}}\left[ t,\tau  \right]={{\left\| {{\overset{\to }{\mathop{r}}\,}_{i}}\left[ t,\tau  \right]-{{\overset{\to }{\mathop{r}}\,}_{j}}\left[ t,\tau  \right] \right\|}_{2}},\forall i,j\in M.
\end{equation}
The distance from the AUV to the eavesdropper is denoted as:
\begin{equation}
    {{d}_{md}}\left[ t,\tau  \right]={{\left\| {{\overset{\to }{\mathop{r}}\,}_{m}}\left[ t,\tau  \right]-{{\overset{\to }{\mathop{r}}\,}_{d}} \right\|}_{2}}.
\end{equation}
Suppose that the movement speed of the AUV is ${\overset{\to }{\bm{v}}\,}_{m}$, its current position can be updated according to the time slot:
\begin{equation}
    {{\overset{\to }{\mathop{r}}\,}_{m}}\left[ t,\tau  \right]={{\overset{\to }{\mathop{r}}\,}_{m}}\left[ t,\tau -1 \right]+\Delta\tau {{\overset{\to }{\mathop{v}}\,}_{m}}.
\end{equation}
\subsection{Underwater Channel Model}
Underwater acoustic waves are currently the most effective means of underwater wireless communication and detection \cite{UAC01} \cite{UAC02}. However, the characteristics of underwater acoustic channels are extremely complex, exhibiting significant signal attenuation, pronounced multipath effects, considerable propagation delays, and limited available bandwidth. To effectively evaluate system performance, a reasonable channel model is crucial. In our work, we employ an empirical model based on Thorp's theory to characterize the propagation loss of underwater acoustic signals. The Thorp model considers two primary factors that affect signal strength: spreading loss and absorption loss. \textit{Spreading loss} describes the geometric divergence of energy as sound waves propagate through the medium, typically proportional to a power of the propagation distance. \textit{Absorption loss}, on the other hand, arises from additional signal energy attenuation due to energy conversion during propagation through seawater.

Specifically, the absorption coefficient $\alpha \left( f \right)$ (dB/km) of an acoustic signal propagating through water at a frequency $f$ (KHz) can be approximately calculated using Thorp's empirical formula:
\begin{equation}
    10lg(a\left( f \right))=\frac{0.11{{f}^{2}}}{1+{{f}^{2}}}+\frac{44{{f}^{3}}}{4100+{{f}^{2}}}+2.75\times {{10}^{-4}}{{f}^{2}}+0.003.
\end{equation}
The total path loss factor $A(f,d)$ of the signal after propagating a distance $d$ is determined by both spreading loss and absorption loss. Therefore, the path loss can be expressed as:
\begin{equation}
    A\left( f,d \right)={{d}^{\chi }}a{{\left( f \right)}^{d}}.
\end{equation}
Therefore, the power gain of the channel at frequency $f$ and distance $d$ can be expressed as the inverse of the path loss:
\begin{equation}
    {{g}_{i,j}}\left[ t,\tau  \right]\!=\!\frac{1}{A\left( f,{{d}_{i,j}}\left[ t,\tau  \right] \right)}\!=\!\frac{1}{{{\left( {{d}_{i,j}}\left[ t,\tau  \right] \right)}^{\chi }}a{{\left( f \right)}^{{{d}_{i,j}}\left[ t,\tau  \right]}}},
\end{equation}
where ${{d}_{i,j}}\left[ t,\tau  \right]$ represents the distance between AUVs $i$ and $j$ in micro time slot $\tau$.

The quality of the received signal is also affected by complex and varied background noise. Unlike many terrestrial wireless channels, where noise can be approximated as Gaussian white noise, underwater ambient noise typically exhibits significant frequency dependence. According to the underwater noise model in~\cite{noise_model}, the marine environmental noise consists of four components: turbulence ($N_t$), surface waves ($N_s$), shipping ($N_w$), and thermal noise ($N_{th}$), which are mathematically defined as follows:
\begin{equation}\label{noise}
\begin{aligned}
&10\lg {N}_{t}(f) \!=\! 17 - 30\lg f, \\
&10\lg {N}_{s}(f) \!=\! 30 + 20s + 26\lg f - 60\lg(f + 0.03), \\
&10\lg {N}_{w}(f) \!=\! 50 \!+\! 7.5w^{1/2} \!+\! 20\lg f \!-\! 40\lg(f \!+\! 0.4), \\
&10\lg {N}_{th}(f) \!=\! -15 + 20\lg f.
\end{aligned}
\end{equation}
where $f$ is the frequency in kHz, $s$ denotes the shipping activity factor. $N_t(f)$, $N_s(f)$, $N_w(f)$, and $N_{th}(f)$ represent the turbulence, shipping activity, wind-driven waves, and thermal noise, respectively.
The overall power spectral density of the noise can thus be expressed as
$N\left( f \right)={{N}_{t}}\left( f \right)+{{N}_{s}}\left( f \right)+{{N}_{w}}\left( f \right)+{{N}_{th}}\left( f \right)$.
\vspace{-2mm}
\subsection{Received Signal Model}
The received signal of AUV $m$ from the central AUV is:
\begin{equation}
    {{y}_{m}}\left[ t,\tau  \right]\!=\!
    \left\{ 
\begin{aligned}
  & {{{G}_{m}}\left[ t \right]\sqrt{{{P}_{m}}\left[ t,\tau  \right]{{g}_{m,AUV}}\left[ t,\tau  \right]}}{{s}_{m}}\!+\!{{n}_{m}},\;\;\;\;\;{\mathcal{K}_{1}}, \\ 
 & {{n}_{m}},\;\;\;\;\;\;\;\;\;\;\;\;\;\;\;\;\;\;\;\;\;\;\;\;\;\;\;\;\;\;\;\;\;\;\;\;\;\;\;\;\;\;\;\;\;\;\;\;\;\;\;\;\;\;\;\;\;\;\;{\mathcal{K}_{0}}, \\ 
\end{aligned} \right.
\end{equation}
where, $P_m$ and $s_m$ represent the transmission power and transmitted signal of AUV $m$, respectively.  $\mathcal{K}_1$ indicates active communication and $\mathcal{K}_0$ indicates the AUV is silent. We denote the distribution of $y_d$ under hypotheses $\mathcal{K}_0$ and $\mathcal{K}_1$ as $\mathcal{H}_0$ and $\mathcal{H}_1$, respectively.
The signal model for the eavesdropper $D$ is:
\begin{equation}
    {{y}_{d}}\left[ t,\tau  \right]\!=\!\left\{ 
    \begin{aligned}
  & \sum\limits_{m\in M}{{{G}_{m}}\left[ t \right]\sqrt{{{P}_{m}}\left[ t,\tau  \right]{{g}_{m,d}}\left[ t,\tau  \right]}}{{s}_{m}}\!+\!{{n}_{d}},\;\;\;{\mathcal{K}_{1}}, \\ 
 & {{n}_{d}},\;\;\;\;\;\;\;\;\;\;\;\;\;\;\;\;\;\;\;\;\;\;\;\;\;\;\;\;\;\;\;\;\;\;\;\;\;\;\;\;\;\;\;\;\;\;\;\;\;\;\;\;\;\;\;\;\;\;\;\;\;{\mathcal{K}_{0}}. \\ 
\end{aligned} \right.
\end{equation}

In each micro time slot, participating AUVs transmit signals to the central AUV. Since the signals $s_m$ from different AUVs are non-coherent at the eavesdropper, the total received signal is a linear combination of individual signal components and noise. where $s_m$, $n_m$, and $n_d$ follow the Gaussian distribution of $s_m \sim \mathcal{N}(0,1)$, $n_m \sim \mathcal{N}(0,N_m)$, $n_d \sim \mathcal{N}(0,N_d)$ , respectively. The SNR for the eavesdropper is defined as:
\begin{equation}
    {{\gamma }_{d}}\left[ t,\tau  \right]=\sum\limits_{m\in \mathcal{M}}{\frac{G_{m}^{2}\left[ t \right]{{P}_{m}}\left[ t,\tau  \right]}{{{A}_{m,d}}\left[ t,\tau  \right]{{N}_{d}}}},
\end{equation}
and the data transmission rate between AUVs $m$ and $n$ is:
\begin{equation}
    {{R}_{m,n}}\!\left[ t,\tau  \right]\!\!=\!B{{\log }_{2}}\!\left( 1\!+\!\frac{{{\sum\limits_{m,n\in M}{G_{m}^{2}\left[ t \right]{{P}_{m}}\left[ t,\tau  \right]{{g}_{m,n}}\left[ t,\tau  \right]}}}}{N_n\left( f \right)} \right)\!.
\end{equation}

\subsection{Covertness Model}
For the considered multi-AUV system, we use LRT theory to model the eavesdropper's listening behavior. Since the log-likelihood ratio has a linear monotonic relationship with the energy of the received signal, changing the detection method can be equivalently reduced to a computationally more efficient energy detection scheme. Therefore, for the detection process, we propose directly comparing the energy of the received signal with a threshold value, as stated in the following lemma. 

\begin{lemma}
    Considered a received signal $y_d$ that is modeled as a zero-mean real Gaussian random variable under both the null hypothesis $\mathcal{K}_0$ and the alternative hypothesis $\mathcal{K}_1$, with variances $\delta_0^2$ and $\delta_1^2$ respectively, where $\delta_1^2 > \delta_0^2$. The likelihood ratio test (LRT) used to distinguish between $\mathcal{K}_0$ and $\mathcal{K}_1$ is equivalent to an energy detector of the form:
    \begin{equation}
        {{\left| {{y}_{d}} \right|}^{2}}\succ {{\Theta }^{'}}.
    \end{equation}
    The optimal threshold $\Theta'$ for the energy detector is given by: ${{\Theta }^{'}}=\frac{2\ln \Theta -\ln \left( \frac{\sigma _{0}^{2}}{\sigma _{1}^{2}} \right)}{\frac{\sigma _{1}^{2}-\sigma _{0}^{2}}{\sigma _{0}^{2}\sigma _{1}^{2}}}.$
\end{lemma}
\begin{proof}
    We first derive the likelihood ratio function of the received signal $y_d$:
\begin{equation}
    L\left( {{y}_{d}} \right)=\frac{p\left( {{y}_{d}}|{\mathcal{K}_{1}} \right)}{p\left( {{y}_{d}}|{\mathcal{K}_{0}} \right)}.
\end{equation}
Here, $p\left( {{y}_{d}}|{\mathcal{K}_{1}} \right)$ denotes the probability density function (PDF) of a real Gaussian distribution, given by $p(z)\sim \mathcal{N}(0,\sigma^2).$ By substituting the probability density functions (PDFs) of the received signal $y_d$ under hypotheses $\mathcal{K}_0$ and $\mathcal{K}_1$ into the above expression, we can obtain:
\begin{equation}
\begin{aligned}
  & L\left( {{y}_{d}} \right)=\frac{\frac{1}{\sqrt{2\pi \sigma _{1}^{2}}}{{e}^{\left( -\frac{y_{d}^{2}}{2\sigma _{1}^{2}} \right)}}}{\frac{1}{\sqrt{2\pi \sigma _{0}^{2}}}{{e}^{\left( -\frac{y_{d}^{2}}{2\sigma _{0}^{2}} \right)}}}=\sqrt{\frac{\sigma _{0}^{2}}{\sigma _{1}^{2}}}{{e}^{\left( {{\left| {{y}_{d}} \right|}^{2}}\left( \frac{1}{2\sigma _{0}^{2}}-\frac{1}{2\sigma _{1}^{2}} \right) \right)}}. \\ 
\end{aligned}
\end{equation}
Let the energy detection threshold be denoted by $\Theta$. Taking the logarithm on both sides of the equation, we can obtain:
\begin{equation}
    \ln \left( \frac{{{\sigma }_{0}}}{{{\sigma }_{1}}} \right)+{{\left| {{y}_{d}} \right|}^{2}}\left( \frac{1}{2\sigma _{0}^{2}}-\frac{1}{2\sigma _{1}^{2}} \right)\gtrless _{{\mathcal{K}_{0}}}^{{\mathcal{K}_{1}}}\ln \left( \Theta  \right).
\end{equation}
Since $\delta_1^2 > \delta_0^2$, the coefficient of ${{\left| {{y}_{d}} \right|}^{2}}$ is positive, so the function is monotonically increasing with respect to ${{\left| {{y}_{d}} \right|}^{2}}$. Therefore, we can obtain: ${{\left| {{y}_{d}} \right|}^{2}}\succ {{\Theta }^{'}}$, ${{\Theta }^{'}}=\frac{2\ln \Theta -\ln \left( \frac{\sigma _{0}^{2}}{\sigma _{1}^{2}} \right)}{\frac{\sigma _{1}^{2}-\sigma _{0}^{2}}{\sigma _{0}^{2}\sigma _{1}^{2}}}.$\end{proof}

The total detection error rate of the eavesdropper in time slot $[t, \tau]$ consists of the false alarm probability $P_{\mathrm{FA}}[t, \tau]$ and the missed detection probability $P_{\mathrm{MD}}[t, \tau]$: $\xi ={{P}_{MD}[t, \tau]}+{{P}_{FA}[t, \tau]}$.
To ensure covert communication, the total detection error rate must satisfy $\xi \geq 1 - \varepsilon$, where $\varepsilon$ is a predefined, very small parameter representing the requirement for covertness.
However, directly computing $P_{\mathrm{FA}}$ and $P_{\mathrm{MD}}$ and incorporating them as optimization constraints is extremely difficult.
Therefore, we adopt an information-theoretic approach to transform this probabilistic constraint into a more tractable form in the following lemma.

\begin{lemma}
To satisfy the covertness requirement $\xi \ge 1-\varepsilon $, a stricter KL divergence constraint $D\left( {\mathcal{H}_{0}}||{\mathcal{H}_{1}} \right)$ can be imposed. This KL divergence can be precisely expressed in terms of the total received signal-to-noise ratio (SNR) at the eavesdropper as:
\begin{equation} \label{covert1}
    D\left( {\mathcal{H}_{0}}||{\mathcal{H}_{1}} \right)=\frac{1}{2}\left(\ln \left( 1+{{\gamma }_{d}}\left[ t,\tau  \right] \right)-\frac{{{\gamma }_{d}}\left[ t,\tau  \right]}{1+{{\gamma }_{d}}\left[ t,\tau  \right]} \right).
\end{equation}

\end{lemma}
\begin{proof}
    The covertness constraint requires that the total detection error probability satisfies $\xi ={{P}_{MD}}+{{P}_{FA}} $. For the optimal eavesdropper, the error probability is related to the total variation distance $V(\mathcal{H}_0,\mathcal{H}_1)$:
\begin{equation}
    \xi =1-V\left( {\mathcal{H}_{0}},{\mathcal{H}_{1}} \right).
\end{equation}
Based on $\xi \ge 1-\varepsilon $, the covertness requirement can be equivalently expressed as:
$V\left( {\mathcal{H}_{0}},{\mathcal{H}_{1}} \right)\le \varepsilon .$ 
According to Pinsker's inequality in information theory, the total variation distance has the following upper bound in terms of the Kullback–Leibler divergence $D\left(\mathcal{H}_0||\mathcal{H}_1\right)$:
\begin{equation}
    V\left( {\mathcal{H}_{0}},{\mathcal{H}_{1}} \right)\le \sqrt{\frac{1}{2}D\left( {\mathcal{H}_{0}}||{\mathcal{H}_{1}} \right)}.
\end{equation}
Combining the above two inequalities, we can obtain a stronger constraint: $D\left( {\mathcal{H}_{0}}||{\mathcal{H}_{1}} \right)\le 2{{\varepsilon }^{2}}.$  
Since Pinsker’s inequality provides an upper bound, the condition $D\left( {\mathcal{H}_{0}}||{\mathcal{H}_{1}} \right)\le 2{{\varepsilon }^{2}}$ is more stringent than the condition $V\left( {\mathcal{H}_{0}},{\mathcal{H}_{1}} \right)\le \varepsilon$. Therefore, we adopt the KL divergence as the covertness constraint here. Next, we derive the explicit expression of the KL divergence $D\left( {\mathcal{H}_{0}}||{\mathcal{H}_{1}} \right)$.

According to the definition of the KL divergence, we have:
\begin{equation}
    D\left( {\mathcal{H}_{0}}||{\mathcal{H}_{1}} \right)=\int{p\left( {{y}_{d}}|{\mathcal{K}_{0}} \right)}\log \frac{p\left( {{y}_{d}}|{\mathcal{K}_{0}} \right)}{p\left( {{y}_{d}}|{\mathcal{K}_{1}} \right)}d{{y}_{d}}.
\end{equation}
Similarly, by substituting the PDFs of the received signal under hypotheses $\mathcal{K}_0$ and $\mathcal{K}_1$, we obtain:
\begin{equation}
    D\left( {\mathcal{H}_{0}}||{\mathcal{H}_{1}} \right)=\int_{-\infty }^{+\infty }{\frac{1}{\sqrt{2\pi \sigma _{0}^{2}}}{{e}^{-\frac{y_{d}^{2}}{2\sigma _{0}^{2}}}}}\ln \left( \frac{\frac{1}{\sqrt{2\pi \sigma _{0}^{2}}}{{e}^{-\frac{y_{d}^{2}}{2\sigma _{0}^{2}}}}}{\frac{1}{\sqrt{2\pi \sigma _{1}^{2}}}{{e}^{-\frac{y_{d}^{2}}{2\sigma _{1}^{2}}}}} \right)d{{y}_{d}}.
\end{equation}
\begin{figure*}[ht]
    \centering
    \begin{equation} \label{58}
    \begin{aligned}
   D\left( {\mathcal{H}_{0}}||{\mathcal{H}_{1}} \right)
  &=\int_{-\infty }^{+\infty }{p\left( {{y}_{d}}|{{H}_{0}} \right)}\left[ \frac{1}{2}\ln \left( \frac{\sigma _{1}^{2}}{\sigma _{0}^{2}} \right)+\frac{{{\left| {{y}_{d}} \right|}^{2}}}{2}\left( \frac{1}{\sigma _{1}^{2}}-\frac{1}{\sigma _{0}^{2}} \right) \right]d{{y}_{d}} \\ 
 &=\int{p\left( {{y}_{d}}|{\mathcal{H}_{0}} \right)}\left[ \frac{1}{2}\ln \left( \frac{\sigma _{1}^{2}}{\sigma _{0}^{2}} \right) \right]d{{y}_{d}}
 +\int{p\left( {{y}_{d}}|{\mathcal{H}_{0}} \right)}\left[ \frac{{{\left| {{y}_{d}} \right|}^{2}}}{2}\left( \frac{1}{\sigma _{1}^{2}}-\frac{1}{\sigma _{0}^{2}} \right) \right]d{{y}_{d}} \\ 
 &=\frac{1}{2}\ln \left( \frac{\sigma _{1}^{2}}{\sigma _{0}^{2}} \right)+\frac{1}{2}\left( \frac{\sigma _{1}^{2}}{\sigma _{0}^{2}}-1 \right) =\frac{1}{2}\left( \ln \left( \frac{\sigma _{1}^{2}}{\sigma _{0}^{2}} \right)+\frac{\sigma _{1}^{2}}{\sigma _{0}^{2}}-1 \right) \\ 
\end{aligned}
\end{equation}
    \noindent\rule{1\linewidth}{0.4pt}
\end{figure*}
\begin{figure*}[ht]
    \centering
    \begin{equation} \label{delta_1}
    \begin{aligned}
   \sigma _{1}^{2}
  &=E\left[ {{\left| {{y}_{d}} \right|}^{2}}|{\mathcal{K}_{1}} \right] =E\left[ {{\left| \sum\limits_{m\in \mathcal{M}}{{{G}_{m}}\left[ t \right]\sqrt{{{P}_{m}}\left[ t,\tau  \right]}{{h}_{m,d}}\left[ t,\tau  \right]}{{s}_{m}}\left[ t,\tau  \right]+{{n}_{d}}\left[ t,\tau  \right] \right|}^{2}} \right] \\ 
 &=\sum\limits_{m\in \mathcal{M}}{E\left[ {{\left| {{G}_{m}}\left[ t \right]\sqrt{{{P}_{m}}\left[ t,\tau  \right]}{{h}_{m,d}}\left[ t,\tau  \right]{{s}_{m}}\left[ t,\tau  \right] \right|}^{2}} \right]+E\left[ {{\left| {{n}_{d}} \right|}^{2}} \right]} \\ 
 &=\sum\limits_{m\in \mathcal{M}}{{{G}_{m}}{{\left[ t \right]}^{2}}{{P}_{m}}\left[ t,\tau  \right]}E\left[ {{\left| {{h}_{m,d}}\left[ t,\tau  \right] \right|}^{2}} \right]E\left[ {{\left| {{s}_{m}} \right|}^{2}} \right]+{{N}_{d}} = \sum\limits_{m\in \mathcal{M}}{\frac{{{G}_{m}}\left[ t \right]{{P}_{m}}\left[ t,\tau  \right]}{{{A}_{m,d}}\left[ t,\tau  \right]}} +{{N}_{d}}. \\ 
\end{aligned}
\end{equation}
    \noindent\rule{1\linewidth}{0.4pt}
\end{figure*}
To evaluate the integral, we first simplify the logarithmic term and then split the expression into two parts to compute the integrals separately, as shown in (\ref{58}).

The integral term $\int p\left( y_d \mid K_0 \right) \left| y_d \right|^2 \, dy_d$ is the second-order moment of the random variable $y_d$ under the distribution $\mathcal{K}_0$, which corresponds to its variance $\sigma_0^2$.
Therefore, $\int p\left( y_d \mid K_0 \right) \left| y_d \right|^2 \, dy_d = \sigma_0^2.$
Next, we need to specify the variances of the received signal under hypotheses $\mathcal{K}_0$ and $\mathcal{K}_1$, as well as the relationship between them. 

Under hypothesis $\mathcal{K}_0$, the eavesdropper only receives environmental noise, and thus the variance is equal to the noise power $N_d$: $\delta_0^2 = N_d$.
Under hypothesis $\mathcal{K}_1$, the signal $y_d$ received by the eavesdropper is the superposition of the signals transmitted by all participating AUVs and the noise. Since all transmitted signals $s_m$, channel gain $g_{m,d}$, and noise $n_d$ are zero-mean and mutually independent random variables, the received signal $y_d$ also has zero mean. Its variance $\sigma_1^2$, which represents the total average power of the signal, is given by (\ref{delta_1}).

Thus, we have a relationship: $\frac{\sigma _{1}^{2}}{\sigma _{0}^{2}}={{\gamma }_{d}}\left[ t,\tau  \right]+1$.
By substituting this relationship into equation (\ref{58}), we obtain the final covertness constraint (\ref{covert1}).
\end{proof}

\subsection{Task Model}
The underwater collaborative exploration mission scenario considered in this paper takes the central AUV as the core coordination node, supplemented by multiple AUVs with certain autonomous capabilities for joint execution. The central AUV plays a crucial role in communicating with the upper-level command system (shore-based station) and obtaining macro task instructions.

Specifically, the AUV rises to a position close to the sea surface and establishes a reliable wireless communication link with the BS. Next, the AUV will receive the task packets sent from the BS. This task typically includes core information: the central coordinates of the area to be detected, as well as its length and width. After receiving the task, the central AUV will dive to the predetermined operational depth. Subsequently, the central AUV uses underwater acoustic communication technology to distribute instructions to the AUVs involved in the task. At this stage, the central AUV is not only a "relay" of information, but also, based on the initial perception of the overall situation (such as the locations and remaining energy of each AUV), delegates the appropriate AUV to participate in this task.

We divide the execution of the task into three stages: the task distribution stage, the AUV task execution stage, and the result fusion stage. Among these, the task execution stage is a complex process, comprising a series of sub-activities, such as relocating the assigned AUV from its current position to the target area, performing detection operations using its sensors, and conducting preliminary on-board processing of the collected data.

\subsubsection{Task distribution stage} This stage of the central AUV marks the beginning of the entire collaborative task. Its core lies in that the central AUV effectively transmits the task instructions obtained from the BS to the selected AUV cluster. Specifically, the delay in the task distribution stage is:
\begin{equation}
    T_m^d[t]=\frac{D[t]}{R_{m,cAUV}[t]}+\frac{v_s}{d_{m,cAUV}[t]},
\end{equation}
where $D[t]$ represents the amount of task instruction data issued by central AUV in the macro time slot $t$. $R_{m,cAUV}[t]$ is defined as the data transmission rate between the central AUV and the $m$th AUV within the first micro time slot $\tau$ of the macro time slot $t$. $d_{m,cAUV}[t]$ represents the distance between AUV $m$ and central AUV at the beginning of this macro time slot. Here, $v_s = 1500m/s$ denotes the propagation speed of acoustic waves in water \cite{v_s}. 

\subsubsection{Task execution stage} The first step is to plan the arrival point of each AUV participating in the task within the target area to avoid repetitive exploration. The target area is a rectangular area to be explored composed of the central coordinate, length and width: $\operatorname{tar}[t]=\left(x_{tar}[t], y_{tar}[t], z_{tar}[t]\right), l[t], w[t] .$ Each AUV $m$ has a circular detection radius $r_m$ determined by its own sensor performance and computing ability $C_m$. The exploration radius $r_m$ is linearly correlated with the computing ability and is given by
\begin{equation}
    r_m=r_b+\mu \ln \left(1+\frac{C_m}{C}\right),
\end{equation}
where $r_b$ indicates the basic detection radius, which is determined by the physical performance of its sensor. $\mu$ and $C$ are the coefficient and the reference computing ability threshold. Consequently, we can define the task completion rate as
\begin{equation}
    \zeta[t]=\frac{\sum_{m \in M}G_m[t] \pi r_m^2}{L[t]^2}.
\end{equation}

Then, we need to calculate the total delay experienced by the AUVs in executing the tasks. This delay primarily consists of the time required for a series of actions performed by AUVs, including movement, sensing, and processing. The movement time of the AUVs is significantly influenced by the dynamic characteristics of the complex ocean environment. Therefore, we have developed a dynamic model that accurately represents the actual movement speed of the AUVs under the influence of ocean currents.

Specifically, we utilize a parameterized model based on the superposition of multiple Lamb-Oseen vortices to simulate representative features of mesoscale to small-scale ocean turbulence in three-dimensional space \cite{ocean3D}. This model is combined with a simplified form of the Navier-Stokes equations or their numerical solutions. The horizontal velocity components are recursively updated through the vortex model, while the vertical component is generated based on the horizontal velocities and the covariance matrix. The specific ${{\overset{\to }{\mathop{V}}\,}_{T}}$ is modeled as follows:
\begin{equation}
    \left\{ \begin{aligned}
  & \frac{\partial \varpi }{\partial t}+\left( {{\overset{\to }{\mathop{V}}_{c}}}\nabla  \right)\varpi =h\Delta \varpi , \\ 
 & \varpi \left( \overset{\to }{\mathop{r}}\, \right)=\frac{\beta }{\pi {{l}^{2}}}{{e}^{\frac{-{{\left( \overset{\to }{\mathop{r}}\,-{{\overset{\to }{\mathop{r}_{c}}}} \right)}^{2}}}{{{l}^{2}}}}}, \\ 
 & v_{x}^{C}\left( \overset{\to }{\mathop{r}}\, \right)=-\delta \frac{y-{{y}_{c}}}{2\pi {{\left( \overset{\to }{\mathop{r}}\,-{{\overset{\to }{\mathop{r}_{c}}}} \right)}^{2}}}\left[ 1-{{e}^{\frac{-{{\left( \overset{\to }{\mathop{r}}\,-{{\overset{\to }{\mathop{r}_{c}}}} \right)}^{2}}}{{{l}^{2}}}}} \right], \\ 
 & v_{y}^{C}\left( \overset{\to }{\mathop{r}}\, \right)=-\delta \frac{x-{{x}_{c}}}{2\pi {{\left( \overset{\to }{\mathop{r}}\,-{{\overset{\to }{\mathop{r}_{c}}}} \right)}^{2}}}\left[ 1-{{e}^{\frac{-{{\left( \overset{\to }{\mathop{r}}\,-{{\overset{\to }{\mathop{r}_{c}}}} \right)}^{2}}}{{{l}^{2}}}}} \right], \\ 
 & v_{z}^{C}\left( \overset{\to }{\mathop{r}}\, \right)\!=\!\rho \delta \frac{1}{\sqrt{\det \left( 2\pi \Gamma \right)}}{{e}^{\frac{-{{\left( \overset{\to }{\mathop{r}}\,-{{\overset{\to }{\mathop{r}_{c}}}} \right)}^{T}}}{2\left( \overset{\to }{\mathop{r}}\,-{{\overset{\to }{\mathop{r}_{c}}}} \right)}}},\Gamma\!=\!\left[ \begin{matrix}
   l & 0  \\
   0 & l  \\
\end{matrix} \right]. \\ 
\end{aligned} \right.
\end{equation}

The entire ocean current velocity field is composed of three components: $(v_x^C, v_y^C, v_z^C)$. The horizontal velocity components $(v_x^C, v_y^C)$ are computed based on a two-dimensional vortex model, whose behavior is governed by the spatiotemporal evolution of fluid vorticity $\varpi$. This evolution is jointly influenced by the advection effect of the background current ${{\overset{\to }{\mathop{V}}_{c}}}$ and the dissipative effect described by the fluid viscosity coefficient $h$. In this model, the flow velocity at an arbitrary position ${\overset{\to }r}$ is a complex function of the vortex core ${{\overset{\to }{\mathop{r}_{c}}}}$, characteristic radius $l$, vortex strength $\beta$, and circulation coefficient $\delta$. Considering that vertical motion in ocean environments is typically mild, the vertical velocity component $v_z^C$ is modeled as being driven by the horizontal flow field and coupled via a vertical flow conversion factor $\rho$. Its spatial distribution is described by a Gaussian function with $\Gamma$ as the spatial correlation matrix, ensuring continuity and physical plausibility of the flow field. This integrated model allows us to generate a dynamically evolving 3D current field for simulating the underwater motion of AUVs.
Thus, the velocity of AUV $m$ relative to the ocean current is given by:
\begin{equation}
    \overset{\to }{\mathop{V}}_{m}^{'}\left[ t,\tau  \right]={{\overset{\to }{\mathop{V}}_{m}}}\left[ t,\tau  \right]-{{\overset{\to }{\mathop{V}}_{T}}}.
\end{equation}
Here, ${{\overset{\to }{\mathop{V}}_{m}}}\left[ t,\tau  \right]$ is the velocity generated by the AUV's own propulsion system. 

Considering the large target area and the limited energy and task time of the AUVs, directly directing the AUVs toward the center of the area or employing random wandering for exploration is inefficient. Additionally, the circular detection range of the AUVs does not geometrically align with the rectangular target area, which can lead to unnecessary redundant detections. Therefore, we have developed a greedy strategy to allocate precise arrival positions for each AUV participating in the task.

The exploration area is rectangular, with a length of $l[t]$ and a width of $d[t]$, and $L[t] = l[t]*d[t]$. For AUV $m$, with an exploration radius of $r_m$, its coordinates must satisfy the following constraints: $\forall \left( {{x}_{m}},{{y}_{m}} \right)\in S,{{x}_{m}}\in \left[ {{r}_{m}},l\left[ t \right]-{{r}_{m}} \right],{{y}_{m}}\in \left[ {{r}_{m}},w\left[ t \right]-{{r}_{m}} \right].$

First, the AUVs are sorted based on their exploration radius $r_m$, giving priority to AUVs with larger radius for position allocation in order to reduce the probability of subsequent conflicts. Candidate points are generated using a normal distribution, with the mean located at the center of the rectangle and a standard deviation of $\mu +3\sigma$ to ensure central concentration. This approach aims to position the centers of the circles as close as possible to the center of the area to enhance effective coverage: ${{x}_{m}},{{y}_{m}}\sim N\left( \frac{L}{2},{{\left( \frac{L}{6} \right)}^{2}} \right).$
Whenever a new position is sought, collision detection with the already determined position is required. Among them, the condition for no repeated coverage is that the distance between the centers of the circles is greater than the sum of the radius: $\sqrt{{{\left( {{x}_{i}}-{{x}_{j}} \right)}^{2}}+{{\left( {{y}_{i}}-{{y}_{j}} \right)}^{2}}}\ge {{r}_{i}}+{{r}_{j}}.$ Finally, the greedy algorithm finds the target position for all AUVs.
Then the delay for the AUV to move to its target position is noted as:
\begin{equation}
    \mathbb{I}_m \left[t,\tau\right]=\mathbb{I}\left( \left\| {{\overset{\to }{\mathop{r}}\,}_{m}}\left[ t,\tau  \right]-{{\overset{\to }{\mathop{r}}\,}_{sub,m}}\left[ t \right] \right\|\le {{r}_{m}} \right),
\end{equation}
\begin{equation}
    \tau _{m}^{idx}\left[ t \right]=\min \left\{ \tau _{m}^{idx}\left[ t \right]\in \left\{ 1,..,micro\text{ }steps \right\}| \mathbb{I}_m \left[t,\tau\right]\right\},
\end{equation}
\begin{equation}
    T_{m}^{move}\left[ t \right]=\tau _{m}^{idx}\left[ t \right] \Delta \tau.
\end{equation}

Here, $\tau _{m}^{\text{idx}}[t]$ denotes the index of the first micro time slot in which AUV $m$ meets the arrival condition within the macro time slot $t$, and ${{\vec{r}}_{\text{sub},m}}[t]$ represents the sub-target location assigned to AUV $m$.
When the AUV reaches the target point, it will conduct detection operations in the surrounding area.  As mentioned above, each AUV has an effective circular detection radius $r_m$ determined by its sensor characteristics and on-board processing capabilities. When the AUV reaches the detection position, it will cover the circular area through sensor scanning. The beam angle of the AUV sonar is $\theta$, so the length of the chord covered in one complete rotation is $\varpi =2{{r}_{m}} \sin \left( \theta /2 \right)$, and the time for completing one rotation is given by
\begin{equation}
    T_{m}^{e}=\frac{2\pi }{\varpi }.
\end{equation}

\subsubsection{Result fusion stage} After the AUV completes the detection operation, it must upload the detection result data collected during this macro task to the central AUV for aggregation and processing. Specifically, the data that needs to be uploaded is:
\begin{equation}
    D_{m}^{'}\left[ t \right]=\varphi \left[ t \right] \pi {{r}_{m}}^{2},
\end{equation}
where $\varphi \left[ t \right]$ represents the amount of data collected in a single sample. As a result, the upload delay is:
\begin{equation}
    T_{m}^{up}\left[ t \right]=\frac{D_{m}^{'}\left[ t \right]}{{{R}_{m,cAUV}}\left[ t,\tau  \right]}+\frac{v_s}{d_{m,cAUV}[t,\tau]}
\end{equation}

Finally, the total execution delay required for the entire collaborative AUV team to complete this macro task $t$ depends on the time consumed by the AUV that completes its work: 
\begin{equation}
    {{T}_{task}}\left[ t \right]=\underset{m\in M}{\mathop{\max }}\,\left( T_{m}^{d}\left[ t \right]+T_{m}^{move}\left[ t \right]+T_m^e+T_{m}^{up}\left[ t \right] \right).
\end{equation}
\subsection{Energy Consumption Model}
In resource-constrained underwater environments, energy is the core factor that limits the continuous operational capacity and task endurance range of AUVs. Therefore, this subsection develops a comprehensive energy consumption model aimed at quantifying the energy expenditure incurred by AUVs during task activities. Each AUV $m$ starts with an initial energy reserve denoted as $E_{m}^{init}$ at the beginning of the macro task cycle. During a task execution cycle, the total energy consumption of the AUV consists of the following three main components.

\subsubsection{Mobility energy consumption}
This portion of energy consumption is related to the energy expended by the AUV to overcome current resistance and perform positional transfers in a three-dimensional underwater environment. It primarily consists of three components: horizontal movement energy consumption, vertical movement energy consumption, and fluid drag energy consumption, among which the horizontal movement energy consumption is given by equation (\ref{E_m}).
\begin{figure*}[ht]
    \centering
    \begin{equation}    \label{E_m}
        E_{h}^{m}\left[ t,\tau  \right]=\frac{{{G}^{2}}\Delta\tau }{\sqrt{2}{{\rho }_{l}}A}{{\left( \left( r_{x}^{m}{{\left[ t,\tau  \right]}^{2}}+r_{y}^{m}{{\left[ t,\tau  \right]}^{2}} \right)+{{\left( r_{x}^{m}{{\left[ t,\tau  \right]}^{2}}+r_{y}^{m}{{\left[ t,\tau  \right]}^{2}} \right)}^{2}}+\frac{{{G}^{2}}}{\rho _{l}^{2}{{A}^{2}}} \right)}^{-\frac{1}{2}}}
    \end{equation}
    \hrulefill
\end{figure*}
$G$ represents the weight of the AUV, $\Delta\tau$ is the duration of micro time slots, $A$ is the cross-sectional area in the direction of movement, and $\rho_l$ is the mass density of seawater. The energy consumed by the AUV during descent in the $\tau$-th time slot is:
\begin{equation}
    E_{d}^{m}\left[ t,\tau  \right]=Gv_{z}^{m}\left[ t,\tau  \right]\Delta\tau .
\end{equation}
By applying computational fluid dynamics techniques, the fluid drag experienced by the AUV can be expressed as:
\begin{equation}
    F_{d}^{m}\left[ t,\tau  \right]=\frac{1}{2}{{\rho }_{l}}A{{C}_{d}}\left\| \overset{\to }{\mathop{V}}\,_{m}^{'}\left[ t,\tau  \right] \right\|_{2}^{2},
\end{equation}
where $C_d$ denotes the drag coefficient. Based on this, the energy consumption due to fluid drag in the $\tau$-th time slot can be modeled as:
\begin{equation}
    \begin{aligned}
   E_{f}^{m}\left[ t,\tau  \right]&=F_{d}^{m}\left[ t,\tau  \right]\left\| \overset{\to }{\mathop{V}}\,_{m}^{'}\left[ t,\tau  \right] \right\|_{2}^{2} \Delta\tau  \\ 
 & \text{      =}\frac{1}{2}{{\rho }_{l}}A{{C}_{d}}\Delta\tau \left\| \overset{\to }{\mathop{V}}\,_{m}^{'}\left[ t,\tau  \right] \right\|_{2}^{3}. \\ 
\end{aligned}
\end{equation}

\begin{figure*}[ht]
    \centering
    \begin{equation}
        \begin{aligned}     \label{E_m_sum}
    & E_{m}^{move}\left[ t,\tau  \right]=E_{h}^{m}\left[ t,\tau  \right]+E_{d}^{m}\left[ t,\tau  \right]+E_{f}^{m}\left[ t,\tau  \right] \\ 
     & \text{         =}\frac{\frac{{{G}^{2}}\Delta\tau }{\sqrt{2}{{\rho }_{l}}A}}{\sqrt{\left( r_{x}^{m}{{\left[ t,\tau  \right]}^{2}}+r_{y}^{m}{{\left[ t,\tau  \right]}^{2}} \right)+{{\left( r_{x}^{m}{{\left[ t,\tau  \right]}^{2}}+r_{y}^{m}{{\left[ t,\tau  \right]}^{2}} \right)}^{2}}+\frac{{{G}^{2}}}{\rho _{l}^{2}{{A}^{2}}}}}+Gv_{z}^{m}\left[ t,\tau  \right]\Delta\tau +\frac{1}{2}{{\rho }_{l}}A{{C}_{d}}\Delta\tau \left\| \overset{\to }{\mathop{V}}\,_{m}^{'}\left[ t,\tau  \right] \right\|_{2}^{3} \\ 
    \end{aligned}
    \end{equation}
    \noindent\rule{1\linewidth}{0.4pt}
\end{figure*}
Therefore, the total energy consumption for the movement of AUV $m$ is given by (\ref{E_m_sum}).

\subsubsection{Detection energy consumption}
When the AUV performs its core exploration tasks, the onboard sensor array and related signal processing units consume significant amounts of energy. We refer to this portion of energy consumption as exploration energy consumption, which is related to the area that the AUV needs to cover for exploration and the inherent energy efficiency of the sensors. We model it as $
    E_{m}^{\det }\left[ t \right]={{\varepsilon }_{\det }} \pi r_{m}^{2}$,
where ${\varepsilon }_{\det }$ represents the energy coefficient for detecting a unit area, indicating the average energy required to perform detection per unit area (in square meters).

\subsubsection{Data upload energy consumption}
 The electro-acoustic conversion efficiency of the modulator-demodulator typically ranges from 20\% to 70\% \cite{efficiency}, so the electrical power required at the transmitter to generate a sound wave is given by
\begin{equation}
    p_{m}^{e}\left[ t \right]=\frac{{{p}_{m}}\left[ t,\tau  \right]}{{{\eta }_{e}}}.
\end{equation}
The transmission time is primarily determined by the amount of data to be transmitted, $D_m^{'}[t]$, and the data rate:
\begin{equation}
    t_{m}^{trans}\left[ t \right]=\frac{D_{m}^{'}\left[ t \right]}{{{R}_{m,cAUV}}\left[ t,\tau  \right]}.
\end{equation}
Therefore, the energy consumption for data upload is:
\begin{equation}
    E_{m}^{trans}\left[ t \right]=p_{m}^{e}\left[ t \right] t_{m}^{trans}\left[ t \right].
\end{equation}
The sum energy consumption of AUV $m$ is calculated as
\begin{equation}
    {{E}_{m}}\left[ t \right]=E_{m}^{init}-E_{m}^{move}\left[ t \right]-\sum\limits_{i\in \tau }{E_{m}^{det}\left[ t,i \right]}-E_{m}^{trans}\left[ t \right].
\end{equation}

\subsection{Problem Formulation}
 We formulate an optimization problem to maximize the overall collaboration efficiency of the system while ensuring the covertness of the AUV team, i.e.,
\begin{align}
  & O{{P}_{1}}:\underset{P\left[ t,\tau  \right],V\left[ t,\tau  \right],G\left[ t \right]}{\mathop{\max }}\,\text{ }\eta[t] , \label{origin_problem}\\ 
 & s.t.\;\;\;\text{    }D\left( {\mathcal{H}_{0}}||{\mathcal{H}_{1}} \right)\left[ t,\tau  \right]\le 2{{\varepsilon }^{2}},\tag{\ref{origin_problem}a} \\ 
 & \;\;\;\;\;\;\;\;\text{        }{{P}_{\min }}\le {{P}_{m}}\left[ t,\tau  \right]\le {{P}_{\max }},\forall m\in M, \tag{\ref{origin_problem}b}\\ 
 & \;\;\;\;\;\;\;\;\text{        }{{G}_{m}}\left[ t \right]\in \left\{ 0,1 \right\},\forall m\in M,\tag{\ref{origin_problem}c} \\ 
 & \;\;\;\;\;\;\;\;\text{        }{{\left\| {{\overset{\to }{\mathop{V}}\,}_{m}}\left[ t,\tau  \right]-{{\overset{\to }{\mathop{V}}\,}_{m}}\left[ t,\tau -1 \right] \right\|}_{2}}\le \alpha i,\forall m\in M, \tag{\ref{origin_problem}d}\\ 
 & \;\;\;\;\;\;\;\;\text{        }{{E}_{m}}\left[ t \right]\ge 0,\text{        }\forall m\in M, \tag{\ref{origin_problem}e}
\end{align}
where $
    \eta \left[ t \right]=\frac{\zeta \left[ t \right]}{{{T}_{task}}\left[ t \right]}
$ indicates the collaboration efficiency of the AUV.

To ensure the AUV team is not detected by potential eavesdroppers while executing the task, the system's communication behavior must satisfy the covert constraint (\ref{origin_problem}a). Constraint (\ref{origin_problem}b) and (\ref{origin_problem}d) address limitations during task execution, specifically the constraints on transmission power and the AUVs' mobility abilities, which are subject to the performance limitations of the propulsion system. Constraint (\ref{origin_problem}c) defines the decision variable constraints, indicating which AUVs the central AUV will activate to participate in the task. Finally, constraint (\ref{origin_problem}e) represents the energy constraint for the AUVs, ensuring that the cumulative energy consumption during task execution does not exceed the initial energy reserves.

This scenario presents a long-sequence decision-making problem, for which we consider modeling it as an MDP and solving it using deep reinforcement learning (DRL) methods. However, given that the optimization variables $G$ and $P$,$V$ correspond to action spaces on different scales (where $G$ only requires a single decision at the beginning of the task, while $P$,$V$ necessitate multiple continuous decisions throughout the entire task execution). The traditional DRL algorithm is not applicable in this scenario. Therefore, we plan to develop a two-scale multi-agent reinforcement learning to solve this problem.

\section{The Proposed HMAPPO Framework % of Hierarchical Reinforcement Learning Algorithm for Covert Collaboration
}\label{method}
In response to the underwater multi-AUV covert collaboration optimization problem established in the previous section, its dual-scale decision characteristics and dynamic collaboration among multiple AUVs make it challenging for traditional DRL methods to directly address. As illustrated in Fig. \ref{framework}, we propose a framework that decomposes the overall optimization problem into two interrelated but time-scale sub-problems: the macro decision layer, where the central AUV assigns tasks, and the micro layer, where individual AUVs perform real-time power and trajectory coordination to meet covertness constraints. %This section provides a detailed explanation of the problem transformation based on the MDP, the design of observation and action spaces, the design of reward functions, and the development of reinforcement learning algorithms.

\subsection{MDP Formulation}
In this subsection, we will model the MDP for the two distinct levels of subproblems: macro-level decision-making and micro-level decision-making. Considering the significant differences between these two levels in terms of the availability of decision information, time scales for action execution, and optimization objectives, we will construct separate MDP and POMDP models for each.

\begin{figure}
    \centering
    	\captionsetup{font={small}} 
    \includegraphics[width=1.05\linewidth]{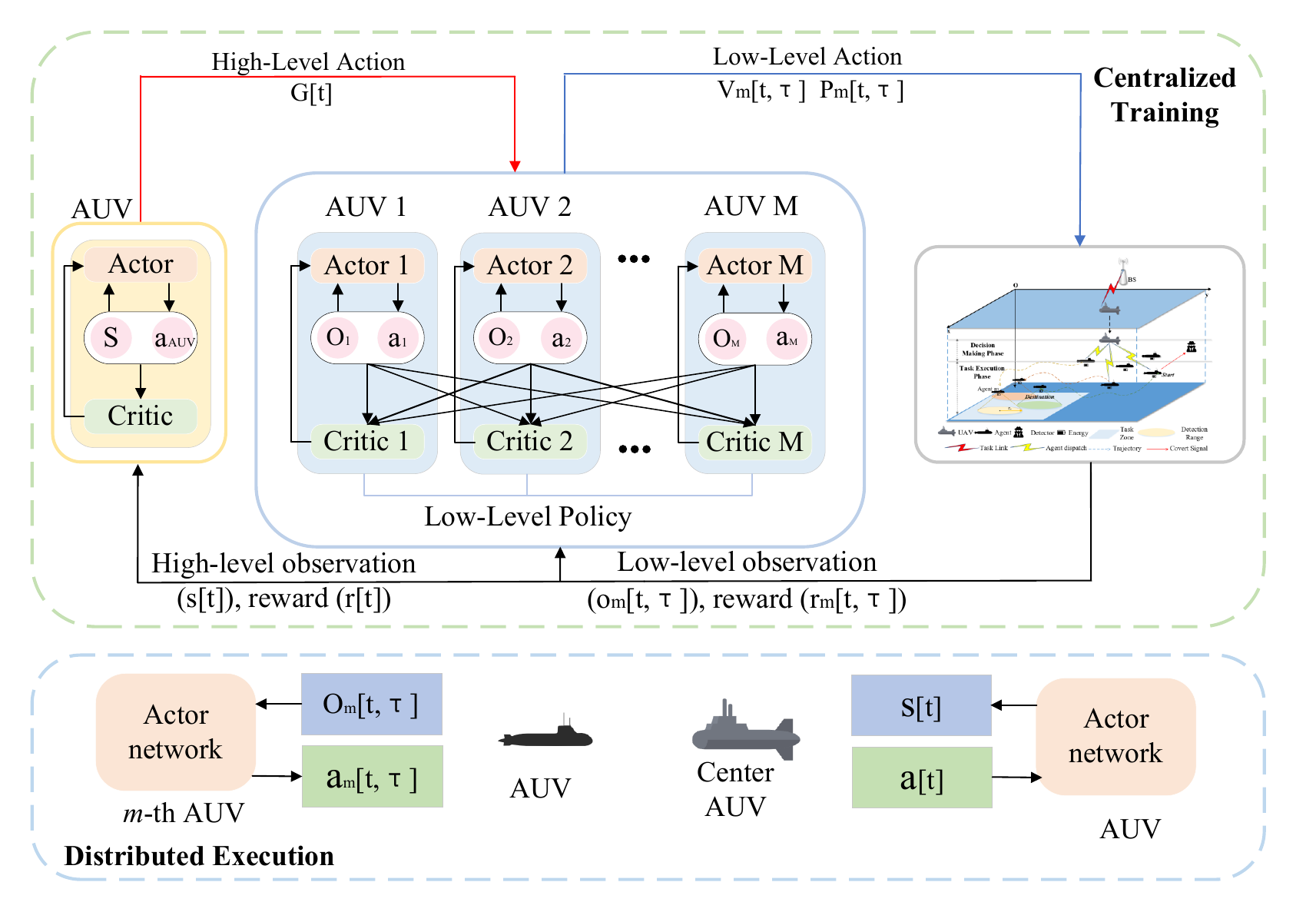}
    \caption{The framework of the proposed HMAPPO.}
    \label{framework}
\end{figure}

%\textbf{Definition 1} (macro-level):
\subsubsection{Macro-level decision-making}
At the macro level of this framework, the central AUV takes on the role of assigning tasks to the AUVs. Its goal is to evaluate the overall system state at the beginning of each macro task cycle and activate the optimal AUVs to maximize the long-term cumulative macro collaboration efficiency. The central AUV can obtain global aggregated information regarding the positions and energy of the AUVs; therefore, we model its decision-making process as an MDP:
\begin{itemize}
    \item \textit{State}: The state space includes the position of the AUV at the beginning of the current task, as well as its remaining energy, i.e.,
    ${{S}_{\rm macro}}\!\triangleq\! \left\{ s_{t,\tau }^{\rm macro} \right\}\!=\!\left\{ \left[ {{\overset{\to }{\mathop{r}}\,}_{m}}\left[ t \right],{{E}_{m}}\left[ t \right] \right] \right\},\forall m\in M.$
    \item \textit{Action}: The action space is defined as a binary decision vector, representing whether the central AUV decides to activate the AUV to participate in the current macro-level task: ${{A}_{AUV}}\triangleq \left\{ \left[ {{G}_{m}}\left[ t \right] \right] \right\}.$
    \item \textit{Reward}: After the completion of the current macro task, our reward function takes into account the task completion rate, task execution time delay, and the average reward per micro timeslot. The goal is to maximize long-term cumulative collaborative efficiency while ensuring the system's sustainability, i.e., ${{R}_{\rm macro}}={{\xi }_{1}}\cdot \zeta \left[ t \right]+{{\xi }_{2}} \cdot {{T}_{task}}\left[ t \right]+{{\xi }_{3}} \cdot \sum\limits_{t\in T}{{{R}_{\rm micro}}\left[ t \right]}/T.$
\end{itemize}

%\textbf{Definition 2} (micro-level):
\subsubsection{Micro-level decision-making}
After the central AUV determines the AUVs participating in the current macro task, each participating AUV transitions to the micro execution layer. At this level, the AUVs need to perform real-time control of their transmission power and navigation speed (trajectory) within continuous micro time slots $\tau$, based on their limited local observation information. Since a single AUV cannot obtain the complete global state, its decision-making process is modeled as a partially observable Markov decision process (POMDP):
\begin{itemize}
    \item \textit{State}: The state space includes the position of AUV, current velocity, noise power and remaining energy: ${{S}_{\rm micro}}\!\triangleq\! \left\{ s_{t,\tau }^{m} \right\}_{m=1}^{M}\!=\!\left\{ \left[ {{\overset{\to }{\mathop{r}}\,}_{m}}\left[ t,\tau  \right],\overset{\to }{\mathop{V}}\,\left[ t,\tau  \right],{{n}_{m}}\left[ t,\tau  \right],{{E}_{m}}\left[ t,\tau  \right] \right] \right\}.$
    \item \textit{Action}: The action space is defined as:
    ${{A}_{\rm micro}}\triangleq \left\{ a_{t,\tau }^{m} \right\}_{m=1}^{M}=\left\{ \left[ {{P}_{m}}\left[ t,\tau  \right],{{V}_{m}}\left[ t,\tau  \right] \right] \right\}.$
    \item \textit{Observation}: Each AUV can only make decisions based on its local observations, which is defined as: ${{O}_{\rm micro}}\triangleq \left\{ o_{t,\tau }^{1},o_{t,\tau }^{2},...,o_{t,\tau }^{M} \right\},$ The observation of AUV $m$ at time slot $\tau$ is described as: 
    \begin{equation}
        o_{t,\tau }^{m}=\left\{ \begin{aligned}
      & {{d}_{m,d}}\left[ t,\tau  \right],{{d}_{m,cAUV}}\left[ t,\tau  \right],{{d}_{m,sub}}\left[ t,\tau  \right] \\ 
     & \overset{\to }{\mathop{{{r}_{m}}}}\,\left[ t,\tau  \right] \\ 
     & \overset{\to }{\mathop{{{V}_{m}}}}\,\left[ t,\tau  \right], \\ 
     & {{G}_{m}}\left[ t \right], \\ 
     & {{E}_{m}}\left[ t,\tau  \right] \\ 
    \end{aligned} \right\}.
    \end{equation}
    \item \textit{Reward}: The micro-level reward design takes into account covertness, task guidance, and energy consumption, aiming to guide the AUV to learn an efficient and economical strategy for completing local tasks while satisfying the covertness constraints.
    The covertness constraint reward is as follows:
    \begin{equation}
        {{r}_{c}}\left[ t,\tau  \right]=\left\{ \begin{aligned}
      & 1,\text{    }D\left( {\mathcal{H}_{0}}||{\mathcal{H}_{1}} \right)\left[ t,\tau  \right]\le 2{{\varepsilon }^{2}}, \\ 
     &-1, \text{   }D\left( {\mathcal{H}_{0}}||{\mathcal{H}_{1}} \right)\left[ t,\tau  \right]\succ 2{{\varepsilon }^{2}}. \\ 
    \end{aligned} \right.
    \end{equation}

    The task completion reward is a sparse but high-value one-time reward, aimed at incentivizing the AUV to complete its core task, which is to reach the sub-target area.
    \begin{equation}
        r_{m}^{\rm task}\left[ t,\tau  \right]=\left\{ \begin{aligned}
      & {{\varpi }_{b}},\text{ }{{d}_{m,sub}}\left[ t,\tau  \right]\prec {{r}_{m}}, \\ 
     & 0,\ \ \ \text{otherwise}. \\ 
    \end{aligned} \right.
    \end{equation}

    Here, when the AUV's distance to the final target is less than its exploration radius, it is considered to have reached the sub-target. Since the reward is sparse, the AUV may lack sufficient guidance signals during the task execution. Therefore, we further design a target-guided reward, which provides denser learning signals and continuously encourages the AUV to approach the sub-target:
    \begin{equation}
        r_{m}^{\rm target}\left[ t,\tau  \right]=\left\{ \begin{aligned}
      & {{\chi }_{p}} \Delta d\left[ t,\tau  \right],\text{    }\Delta d\succ 0, \\ 
     & -{{\chi }_{r}} \Delta d\left[ t,\tau  \right],\text{  }\Delta d\succ 0, \\ 
     & 0,\text{                     }if\text{ }\Delta d=0. \\ 
    \end{aligned} \right.
    \end{equation}

    Here, $\Delta d\left[ t,\tau  \right]={{d}_{m,sub}}\left[ t,\tau -1 \right]-{{d}_{m,sub}}\left[ t,\tau  \right]$ represents the distance change between the current micro time-slot and the previous micro time-slot with with respect to the sub-target.

    When the AUV's energy consumption falls below zero, a penalty proportional to energy consumption falls below zero, a penalty proportional to the energy deficit is applied, aiming to guide the agent to conserve energy: \begin{equation}
        {{r}_{e}}\left[ t,\tau  \right]=\sum\limits_{m\in M}{ReLU\left( -{{E}_{m}}\left[ t,\tau  \right] \right)}.
    \end{equation}

 The sum micro reward is modeled as:
 \begin{align}
      {{R}_{\rm micro}}=&{{\varphi }_{1}}\cdot {{r}_{c}}\left[ t,\tau  \right]+{{\varphi }_{2}}\cdot \sum\limits_{m\in M}{{{G}_{m}} r_{m}^{\rm task}}\left[ t,\tau  \right] \\ 
     & +{{\varphi }_{3}}\cdot \sum\limits_{m\in M}{{{G}_{m}} r_{m}^{\rm target}}\left[ t,\tau  \right]+{{\varphi }_{4}} \cdot{{r}_{e}}\left[ t,\tau  \right].\nonumber  
 \end{align}
\end{itemize}
\subsection{PPO Solution}
To effectively solve the MDP and POMDP modeled in the previous subsection, we employ deep reinforcement learning algorithms based on Proximal Policy Optimization (PPO) and MAPPO. Our proposed hierarchical framework, HMAPPO, leverages these algorithms to address the distinct decision-making challenges at both the macro and micro levels.

\subsubsection{The mechanism of PPO and MAPPO}
At the macro level, the goal of the central AUV is to learn an optimal policy ${{\pi }_{{{\theta }_{cAUV}}}}\left( {{A}_{AUV}}|{{S}_{macro}} \right)$ in order to select the best AUV delegation scheme based on the global state observation $S_{macro}$ to maximize the long-term cumulative macro reward. As an advanced policy gradient method, PPO introduces a clipped surrogate objective function to restrict the magnitude of policy updates, thereby improving training stability while ensuring learning efficiency. 
The objective function of PPO is defined as:
\begin{equation} \label{lclicp}
\begin{aligned}
    &{{L}^{CLIP}}\left( \theta  \right)=\\&{{\mathbb{E}}_{t}}\left[ \min \left( \frac{{{\pi }_{\theta }}\left( a|s \right)}{{{\pi }_{{{\theta }_{old}}}}\left( a|s \right)}{{A}_{t}},clip\left( \frac{{{\pi }_{\theta }}}{{{\pi }_{{{\theta }_{old}}}}},1\!-\!\epsilon ,1\!+\!\epsilon  \right){{A}_{t}} \right) \right].
\end{aligned}
\end{equation}
where, $\frac{{{\pi }_{\theta }}\left( a|s \right)}{{{\pi }_{{{\theta }_{old}}}}\left( a|s \right)}$ represents the ratio of the current policy to the old policy. $A_t$ is the estimated value of the advantage function, which measures the advantage of taking action $a_t$ in state $s_t$; $\epsilon $ is the cropping threshold that controls the magnitude of the policy update.

At the micro level, AUVs must collaborate to control power and trajectory. We modeled a POMDP in previous subsection, whose goal is to learn a policy ${{\pi }_{{{\theta }_{m}}}}\left( {{a}_{m}}|{{o}_{m}} \right)$ that selects the optimal power and speed control action $a$ based on its local observation $o_m$, in order to maximize the long term cumulative micro reward $R_{micro}$ while satisfying constraints such as covertness.

\subsubsection{Estimation of value function}
In order to accurately assess the value corresponding to a state or state-action and to effectively guide strategy learning, both PPO and MAPPO use a value network (Critic). For a critic network at the macro level, the global state observation $S_{macro}$ is input, and the estimate of the value of the state is output: ${{V}_{{{\phi }_{AUV}}}}\left( {{S}_{macro}} \right)$.

In the MAPPO framework, although each AUV's strategy is decentralized, its value function is usually estimated by a central critic network. This critical network has access to the global state information as well as the joint actions of all AUVs, thus evaluating the overall value of the joint behavior. 

\begin{algorithm}[t]
\SetAlgoLined
\small
\KwResult{The trained hierarchical policies: central AUV's actor $\pi_{cAUV}(\theta_{cAUV})$ and AUVs' actors $\{\pi_i(\theta_i)\}_{i=1}^M$.}
Initialize central AUV's network $\pi_{cAUV}(\theta_{cAUV})$,$V_{cAUV}(\phi_{cAUV})$\;
Initialize AUV's networks $\{\pi_i(\theta_i),V_i(\phi_i)\}_{i=1}^M$ with centralized critics;
Initialize experience buffer $\mathcal{B}_{cAUV}$ and $\mathcal{B}_{AUVs}$\;
\For{episode = 1, 2, \ldots, max\_episodes}{
    Reset environment and get initial macro-level state $S_1$\;
    \For{macro-timestep t = 1, 2, \ldots, max\_macro\_steps}{
        Select cAUV schedule $G_t \sim \pi_{cAUV}(S_t; \theta_{cAUV})$\;
        Initialize micro-level environment for task $t$, get initial local observations $\{o_1^i\}_{i=1}^M$\;
        
        \For{micro-timestep $\tau$ = 1, 2, \ldots, max\_micro\_steps}{
            Get global state for centralized critic $s_\tau$\;
            \For{AUV i = 1, 2, \ldots, M}{
                \If{AUV i is selected in $G_t$}{
                    AUV $i$ selects micro-action $a_\tau^i \sim \pi_i(o_\tau^i; \theta_i)$\;
                }
            }
            Execute joint action $\{a_\tau^i\}$, get rewards $\{r_\tau^i\}$ and next local observations $\{o_{\tau+1}^i\}$\;
            Store transition $(s_\tau, \{o_\tau^i\}, \{a_{\tau+1}^i\}, \{r_\tau^i\}, s_{\tau+1})$ into $\mathcal{B}_{AUVs}$\;
        }
        
        Get state $S_{t+1}$ and macro-level reward $R_t$\;
        Store $(S_t, G_t, R_t, S_{t+1})$ into $\mathcal{B}_{AUV}$\;
        
        \If{$\mathcal{B}_{AUVs}$ is ready for update}{
            \For{k = 1, 2, \ldots, training\_epochs}{
                Sample a mini-batch of micro-level transitions from $\mathcal{B}_{AUVs}$\;
                \For{AUV i = 1, 2, \ldots, M}{
                    Compute advantage estimates $\hat{A}_\tau^i$ using GAE according to (\ref{GAE})\;
                    Compute critic loss $L(\phi_i)$ according to (\ref{lcritic})\;
                    Compute actor loss $L^{CLIP}(\theta_i)$ according to (\ref{lclicp})\;
                    Update critic network $\phi_i$ by descending $\nabla_{\phi_i}L(\phi_i)$ according to (\ref{critic})\;
                    Update actor network $\theta_i$ by ascending $\nabla_{\theta_i}L^{CLIP}(\theta_i)$ according to (\ref{actor})\;
                }
            }
            Clear AUV buffer $\mathcal{B}_{AUVs}$\;
        }
        
        \If{$\mathcal{B}_{cAUV}$ is ready for update}{
            Update central AUV's actor and critic networks using data from $\mathcal{B}_{cAUV}$ (via PPO)\;
            Clear central AUV buffer $\mathcal{B}_{cAUV}$\;
        }
    }
}
\caption{The Proposed HMAPPO Algorithm} \label{alg:hmapoo}
\end{algorithm}

\subsubsection{Optimization objectives for critic networks}
Both the macro level central AUV and the micro level AUVs' critic
update their parameters $\phi$ by minimizing the mean squared error (MSE) between their value prediction and the target return $R_t$:
\begin{equation} \label{lcritic}
    L\left( {{\phi }} \right)={{\mathbb{E}}_{t}}\left[ {{{R}_{t}}-{\left( {{V}}\left( {{s}_{t}};{{\phi }} \right) \right)}^{2}} \right].
\end{equation}
where $R_t$ denotes the cumulative discounted return from the current time step $t$ to the end of the future trajectory. In practice, ${\hat{A}_t}$ is typically computed based on the difference between the target return $R_t$ and the current value estimate ${{V}}\left( {{s}_{t}};{{\phi }} \right)$:
\begin{equation} \label{Rt}
    {\hat{A}_t}={{R}_{t}}-{{V}}\left({{s}_{t}};{{\phi }} \right).
\end{equation}
$V$ is the output of the target critic network, and $\phi$ denotes its parameters.

\subsubsection{Computation of the advantage function}
The advantage function A is used to measure the relative quality of taking action $a$ in state $s$, compared to the average action under the current policy. As previously mentioned, we adopt the Generalized Advantage Estimation (GAE):
\begin{equation} \label{GAE}
    {\hat{A}_t}={{\sum\limits_{l=0}^{\infty }{\left( \gamma \lambda  \right)}^{l}}}{{\delta }_{t+l}}.
\end{equation}
Here, the temporal difference (TD) error is defined as:
\begin{equation}
    {{\delta }_{t+l}}={{r}_{t+l}}+\gamma V\left( {{s}_{t+l+1}};\phi  \right)-V\left( {{s}_{t+l}};\phi  \right).
\end{equation}
$L$, $\gamma$, and $\lambda$ denote the time horizon for advantage estimation, the discount factor, and the smoothing parameter in GAE, respectively.

\subsubsection{Policy Improvement and Network Updates}
Based on the computed advantage ${\hat{A}_t}$, the actor networks are improved through policy gradient ascent. The objective for the actors, as shown in equation (\ref{lclicp}) for the central AUV, is to maximize the clipped surrogate objective. Simultaneously, the critic networks are optimized by descending the gradient of their value function loss, which is defined as the mean squared error between the value prediction and the target return $R_t$ from (\ref{Rt}).

The network parameters are updated according to (\ref{lclicp})(\ref{lcritic}). The critic aims to minimize $ L\left( {{\phi }} \right)$, while the actor aims to maximize ${L}^{CLIP}\left( \theta  \right)$. In our implementation, an entropy bonus is also incorporated into the actor's objective to encourage exploration. The update rules for an actor with parameters $\theta$ and a critic with parameters $\phi$ are thus defined as:
\begin{equation}\label{critic}
\phi \leftarrow \phi - \alpha_c \nabla_{\phi} L(\phi),
\end{equation}
\begin{equation} \label{actor}
\theta \leftarrow \theta + \alpha_a \nabla_{\theta} L^{\mathrm{CLIP}}(\theta),
\end{equation}
where $\alpha_a$ and $\alpha_c$ are the learning rates for the critic and actor networks, respectively. 

The detailed procedure of our proposed Hierarchical Multi-Agent Proximal Policy Optimization (HMAPPO) framework is summarized in \textbf{Algorithm \ref{alg:hmapoo}}.

\begin{figure*}[t]
	\centering
    	\captionsetup{font={small}} 
	\includegraphics[width=2\columnwidth]{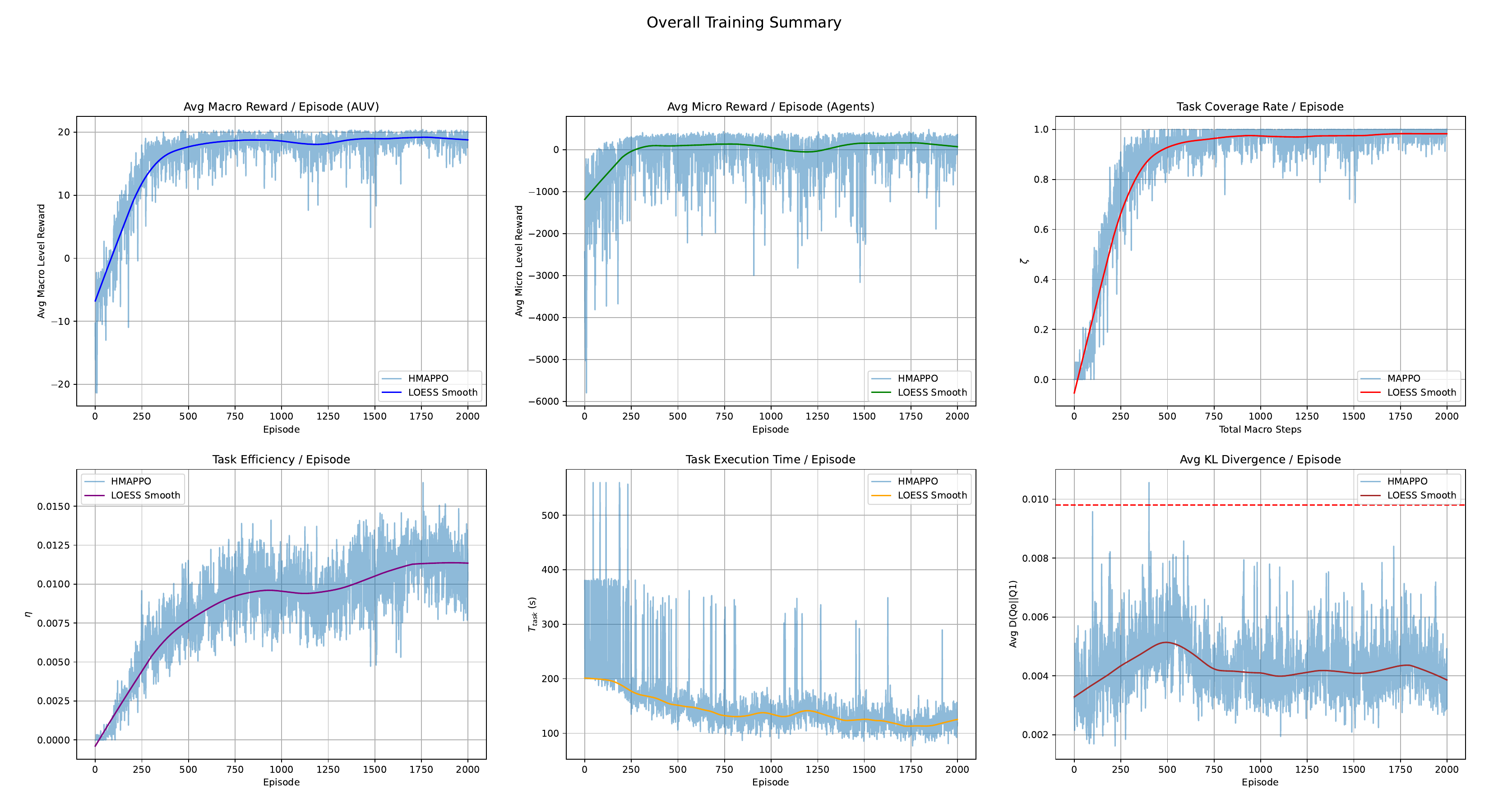}
	\caption{Convergence analysis of the proposed HMAPPO. Subplots show the average episodic trends for: the central AUV's macro-level reward, the AUVs' micro-level reward, the task coverage rate $\zeta$, the task efficiency $\eta$, the task execution time $T_{task}$, and the average KL divergence.}
    \label{train}
\end{figure*}

\subsection{Centralized Training and Decentralized Execution}
Our HMAPPO framework is designed under the CTDE paradigm to achieve both effective training and efficient deployment.

In the \emph{centralized training phase}, we leverage the full observability of the simulation environment. As previously mentioned, the critic networks of all micro-level AUVs have access to the global state information $s_t$ during the learning process. This enables them to learn a comprehensive and accurate value function $V$, which in turn provides a stable and consistent learning signal for all decentralized actors. This mechanism effectively mitigates the inherent non-stationarity issue in multi-AUV systems, where the policies of other AUVs are constantly changing.

In the \emph{decentralized execution phase}, each AUV $m$ (including the central AUV) relies solely on its pre-trained actor network $\pi_\theta(a_i \mid s_i / o_i)$. Based on its local observation (for AUVs) or global observation (for the central AUV), the AUV can make decisions through a single forward propagation. The mechanism ensures that the online decision-making process is fast and requires minimal computational resources.

% \begin{figure*}[!t]
% \centering
% \includegraphics[width=1.8\columnwidth]{net_framework1.pdf}
% \caption{The proposed training framework for our DS-MADRL algorithm.}
% \label{training_frame} 
% \end{figure*}

\section{Numerical Results and Discussion} \label{result}
\subsection{Simulation Setup}
To thoroughly evaluate the effectiveness of the proposed HMAPPO framework and assess its performance advantages over baseline methods, we conducted comprehensive comparative experiments and ablation studies. We built a high-fidelity 3D underwater simulation environment that models a dynamic water body of size $200m \times 200m \times 100m$. The environment includes a central AUV positioned at $(0m,0m,-10m)$. $M$ AUVs randomly deployed in the space, and a eavesdropper located at $(75m, 75m, 5m)$. A current field is generated by a physics-based dynamic model that simulates complex ocean currents, and all communication and energy consumption are computed based on the physical models described earlier. During each training episode, the system is tasked with completing a sequence of $T$ consecutive macro-level tasks, where each task requires AUVs to collaboratively explore a $L \times L$ target region.
The related parameters for our proposed algorithm are presented in Tables \ref{parameters1} and \ref{hyperparams}. 
\begin{table}[thb]
	\centering
    \captionsetup{font={small}} 
	\caption{System parameters}
	\label{parameters1}
	\small
	\begin{tabular}{l | l || l | l}
		\noalign{\global\arrayrulewidth=0.3mm}
		\hline
		\textbf{Parameter}  &\textbf{Value } &\textbf{Parameter}  & \textbf{Value } \\
		%\noalign{\global\arrayrulewidth=0.1mm}
		\hline
		 $M$   & 6 &  $E^{init}$   &$10000\sim 20000$\\
		 $V_{max}$  &5m/s &   $P_{max}$&  $2$\\
         $r_b$  &5  &    $C$&   10 \\
		$A$  & 0.1 & $C_d$ & 0.8 \\
        $\Delta \tau$     &2s & $L$        &30 m\\   
		 $\eta_e$   &0.5  & $B$    &     10 MHz\\   
        ${\chi}$ &1.5  &  $f$ &    30 kHz\\
            $\lambda$ &0.7&   $N_0$ &0.2 W \\
        $\varepsilon$   &0.05    &  $r_b$ &   5 m\\
            $\mu$ &10&   $C_m$ &5 \\
		\noalign{\global\arrayrulewidth=0.3mm}
		\hline
	\end{tabular}
\end{table}	

\begin{table}[t]
\centering
\captionsetup{font={small}} 
\caption{Hyperparameters settings}
\label{hyperparams}
\small
\begin{tabular}{l | c || l | c}
\noalign{\global\arrayrulewidth=0.3mm}
\hline
\textbf{Parameter} & \textbf{Value} & \textbf{Parameter} & \textbf{Value} \\
\hline
Episodes & 2000 & Macro Steps & 10 \\
Micro Steps & 100 & $\alpha_{\text{a}}$ & $3 e{-5}$ \\
$\alpha_{\text{c}}$ & $5 e{-5}$ & $\gamma$ & 0.99 \\
$\epsilon_{\text{clip}}$ & 0.2 & PPO Epochs & 8 \\
Batch Size(AUV) & 512 & Batch Size(cAUV) & 16 \\
Update(AUV) & 2048 & Update(cAUV) & 32 \\
Hidden Size(Actor) & 384 & Hidden Size(Critic) & 512 \\
Hidden Size(cAUV) & 256 & $\lambda_{\text{GAE}}$ & 0.95 \\
Entropy Coeff. & 0.01 & & \\
\noalign{\global\arrayrulewidth=0.3mm}
\hline
\end{tabular}
\end{table}

\begin{figure*}[t]
	\centering
    	\captionsetup{font={small}} 
	\includegraphics[width=2.05\columnwidth]{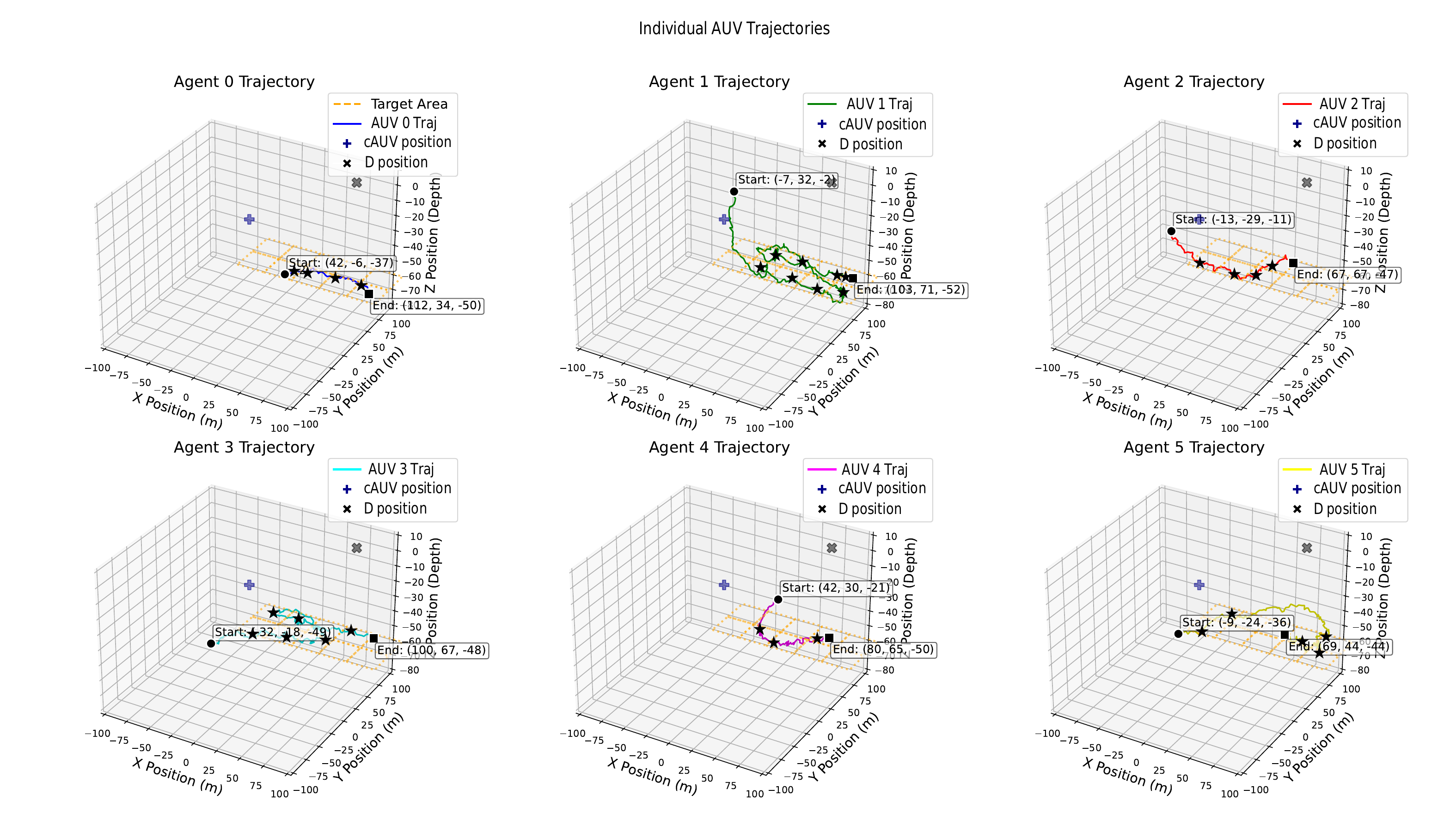}
	\caption{Individual AUV trajectories during the final training episode. The black stars mark the exact position where the AUVs first reached their sub-target areas.}
    \label{trajectory}
\end{figure*}

\subsection{Convergence Performance}
\begin{figure}
    \centering
    	\captionsetup{font={small}} 
    \includegraphics[width=1\linewidth]{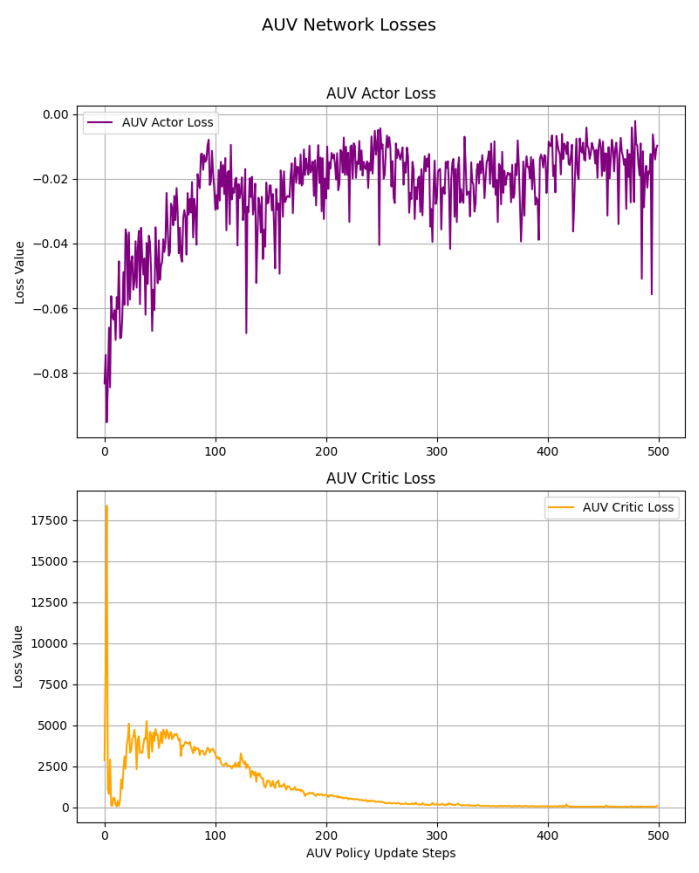}
    \caption{The central AUV critic and actor loss.}
    \label{loss}
\end{figure}
We begin by evaluating the convergence properties of the proposed HMAPPO framework. Fig.~\ref{train} illustrates the evolution of six key performance metrics during training. As shown in Fig.~\ref{train}(b), the average macro-level reward of the central AUV increases significantly from approximately $–20$ to $18$ within the first 500 training iterations. It then stabilizes around $20$ in the later stages of training.
This trend suggests that the central AUV has gradually developed an effective policy for making macro-level decisions.
Correspondingly, the AUVs at the micro level demonstrate a similar convergence trend. Their average micro-level reward improves significantly, rising from an initial value of approximately $–5000$ to a stable region close to zero. This indicates that the AUVs have successfully learned fine-grained control strategies. These strategies enable them to complete their navigation tasks with minimal cost while still satisfying the covertness constraints.

The task coverage ratio Fig.~\ref{train}(c) and task efficiency Fig.~\ref{train}(d) serve as two core metrics for evaluating task completion quality. Both metrics further validate the effective convergence of the proposed algorithm. The coverage rate $\zeta$ increases rapidly during the early stages of training and saturates at nearly 100\% after approximately 500 iterations. This indicates that the AUV teams selected by the central AUV are able to collaboratively achieve near-complete coverage of the exploration region. Our optimization objective, $\eta$, also exhibits a smooth upward trend and eventually converges. In addition, the improvement in task efficiency is attributed not only to better area coverage but also to the system's optimization of task completion time. The average task duration decreases significantly—from over $400$ seconds initially to around $120$ seconds—resulting in an overall latency reduction of approximately 70\%. This demonstrates that our framework can effectively guide the AUV team to reach their targets quickly, accurately, and with low energy consumption, while avoiding task failure and energy depletion caused by being trapped in local optimal but inefficient trajectories.

 \begin{figure*}[t]
	\centering
    	\captionsetup{font={small}} 
	\includegraphics[width=2\columnwidth]{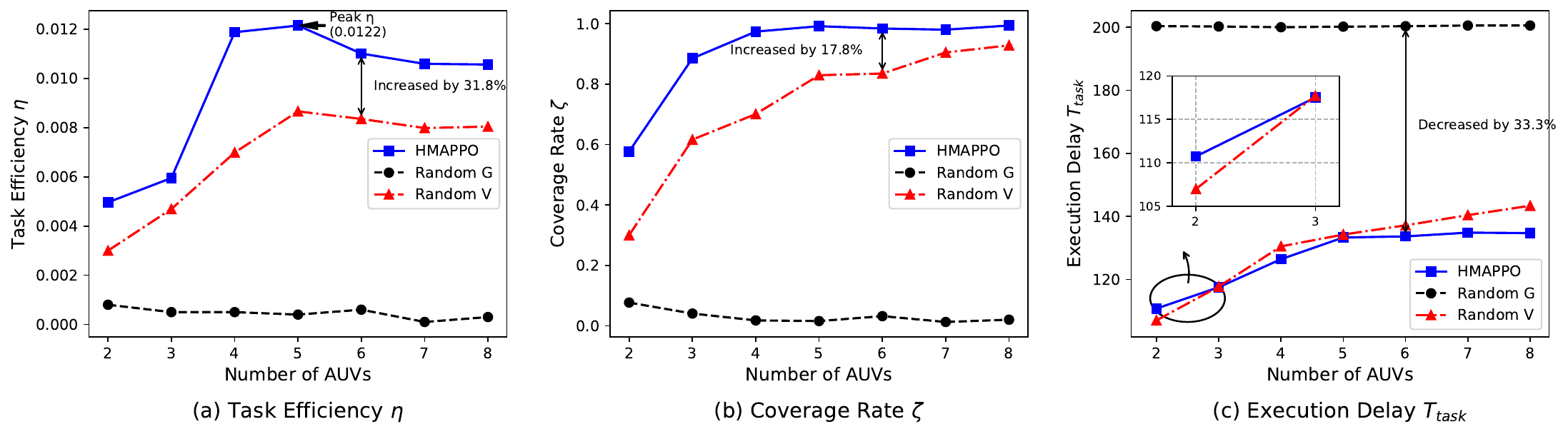}
	\caption{Comparison of collaborative performance under increasing numbers of AUVs in terms of task efficiency, task completion rate, and task execution time, respectively.}
    \label{agent}
\end{figure*}

 \begin{figure*}[t]
	\centering
    	\captionsetup{font={small}} 
	\includegraphics[width=2\columnwidth]{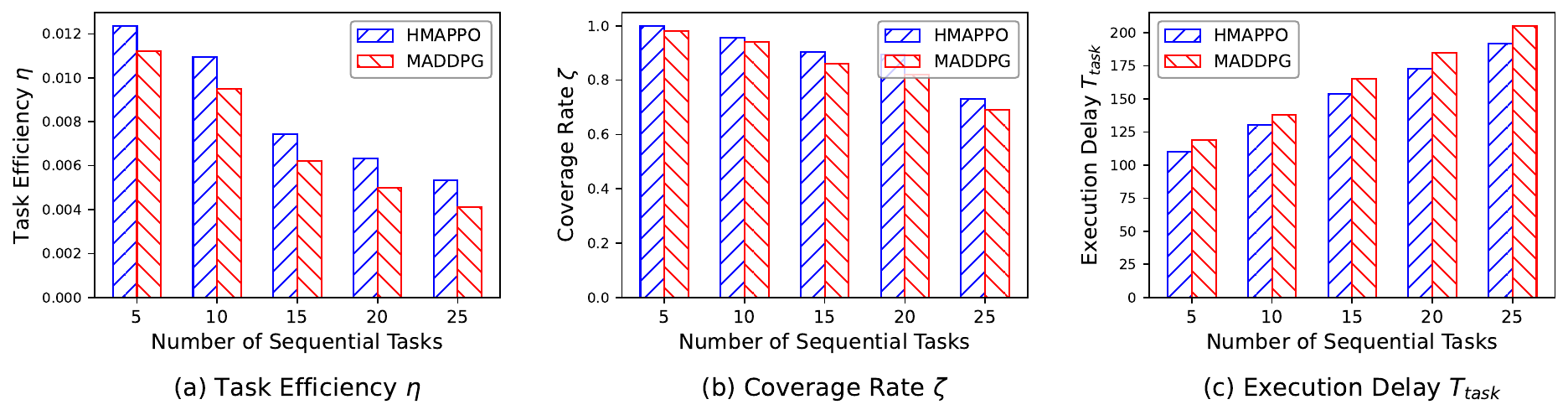}
	\caption{Comparison of collaborative performance under increasing task loads in terms of task efficiency, task completion rate, and task completion time, respectively.}
    \label{task}
\end{figure*}

At the macro level, the central AUV selects and dispatches well-conditioned AUVs through centralized control, thereby ensuring the team's ability to respond quickly and execute tasks efficiently. Fig.~\ref{train}(f) presents the KL divergence curve, which quantitatively demonstrates that our framework can optimize task performance while satisfying communication covertness constraints. During training, the average KL divergence induced by all AUVs' communication behavior is effectively constrained below the predefined threshold of covertness. Even in the early stages of training, when the policy network is still exploring, the peak values of the KL divergence rarely exceed this constraint. Once the policy converges, the average KL divergence stabilizes at a low level of approximately $0.004$. This indicates that our algorithm successfully internalizes the covertness constraint into the learned policy, enabling AUVs to communicate only when necessary and with appropriate power levels.
Meanwhile, both the actor and critic loss curves of the central AUV in Fig.~\ref{loss} exhibit healthy convergence trends: the critic loss rapidly decreases from an initially high value and stabilizes, while the actor loss remains steady within a negative range around $-0.02$, indicating effective learning under the PPO algorithm.

The proposed hierarchical framework achieves stable performance convergence through effective coordination between macro-level decision-making and micro-level execution. Specifically, the convergence of macro-level rewards depends on the successful task execution by micro AUVs, while the convergence of micro-level rewards relies on the central AUV's effective decision-making based on global state information. This two-tier optimization mechanism enables the system to strike an efficient balance among multiple interdependent variables. Ultimately, the simulation results demonstrate that the framework can guide the AUV team to accomplish the assigned tasks efficiently and completely, while adhering to the constraints of a complex dynamic environment and covert communication requirements.

\subsection{AUV Trajectory Analysis}
To intuitively understand the behavioral strategies learned by the AUVs after convergence, Fig.~\ref{trajectory} visualizes the 3D trajectories of AUVs participating in tasks during the final training episode, across $T$ macro task cycles. The figure clearly shows that the AUVs do not follow simple straight-line paths. Instead, they generate complex and smooth curves to reach their respective target points. These nonlinear trajectories reflect the advanced navigation strategies that the AUVs have learned to adapt to and exploit the complex dynamic ocean currents in their environment. By dynamically adjusting their thrust direction and magnitude, the AUVs can overcome or leverage water flow in a way that minimizes both energy consumption and task completion time.

As illustrated in the Fig.~\ref{trajectory}, all AUV trajectories exhibit strong goal-directed behavior. Each AUV successfully navigates from its initial position to the designated target region assigned by the central AUV, without any signs of drifting or inefficient wandering. This result also provides an intuitive exploration for the nearly %100\%
 coverage rate of the zeta metric shown in Fig.~\ref{train}(c). The AUVs' precise and efficient path planning and execution at the micro level serve as the fundamental guarantee for accomplishing collaborative tasks at the macro level.

 \subsection{System Scalability}
 To ensure that the multi-AUV system can cope with scale changes, this section systematically evaluates the system's scalability in terms of both the number of AUVs and the scale of tasks. Fig.~\ref{agent} shows the performance of the system when the number of AUVs ranges from 2 to 8.
 For valid comparisons, we designed two baseline methods that were used simultaneously for scalability assessment and analysis of ablation experiments.
 \begin{itemize}
     \item \textbf{Random decision making G}: This method simulates the scenario without AUV scheduling optimization, i.e., the macro level central AUV randomly selects an available AUV to participate in the task.
     \item \textbf{Random AUV velocity V}: This method simulates the lack of precise speed control of AUVs at the micro level, where their movement velocity $V$ is randomly generated.
 \end{itemize}
The simulation result clearly shows that our proposed HMAPPO framework significantly outperforms the above two baseline methods in all indicators. The optimization objective $\eta$ improves the system performance by (31.8\%, 429.73\%) compared with the two baseline methods, which fully verifies the necessity of collaborative optimization between macro-scheduling and micro-execution: Without global coordination and task assignment at the macro level, the optimization objectives of the micro-level policies can become fragmented or even conflicting, leading to a decline in overall system efficiency. Conversely, without accurate execution of macro-level strategies at the micro level, task performance and energy efficiency cannot be guaranteed, ultimately undermining the overall effectiveness of the system.

In addition, the HMAPPO framework exhibits adaptive behavior as the number of AUVs increases. The task completion rate $\eta$ reaches its peak (approximately 0.012) when the number of AUVs is $4$. However, as the number increases to $8$, $\eta$ remains stable or even slightly decreases. This may reflect a higher-level emergent intelligence: the central AUV has learned to optimize resource allocation. When sufficient AUVs are available, the central AUV does not indiscriminately select all units. Instead, it intelligently selects a sub-team of appropriate size. For example, in the current task sequence, it typically selects around 4–5 AUVs to participate. Since the coverage rate $\zeta$ already approaches $100\%$, adding more AUVs yields diminishing marginal benefits in coverage improvement. On the contrary, it may slightly increase the task execution time due to higher coordination complexity. By effectively decoupling macro-level team formation from micro-level task execution, the system consistently selects near-optimal configurations for task execution, demonstrating strong robustness and scalability.

To more comprehensively evaluate the scalability of our proposed HMAPPO framework, we designed a comparative experiment against another DRL algorithm. Specifically, we explored the performance of task efficiency $\eta$, task completion rate $\zeta$, and execution delay for collaborating on a series of $(5-25)$ predefined tasks when the number of AUVs $M=6$. The baseline  {MADDPG} algorithm \cite{MADDPG} adopts the CTDE architecture and constructs a central critic for each AUV to access global information and evaluate the value of joint actions during the training process; the AUV's actor makes decisions based on local observations during execution.

Fig.~\ref{task} illustrates the performance of the two algorithms under different task loads, in terms of the composite metric task completion rate $\eta$. Our HMAPPO algorithm shows overall superiority under all task loads. Compared with MADDPG, ours $\eta$ always stays ahead, and this advantage is significant with the increase of task load, and the performance improvement extends from about $25\%$ at low task load to nearly $40\%$ at high task load.
Moreover, it can be observed that the $\eta$ decreases gently as the task load increases. At a low task load of $5$ tasks, the system achieves an average task completion rate of $99.8\%$. The system exhibits a plateau in the medium-to-high task load region, ranging from $ 10$ to $20$ tasks. It is possible that our framework has learned to strategically reserve resources in the early stages of a task, sacrificing a small amount of task coverage to avoid premature energy exhaustion of critical AUVs. Until the task load reaches $25$, all the AUVs are exhausted, and $\zeta$ shows a significant drop. In contrast, MADDPG declines more rapidly during the high task load phase, possibly due to its weaker long-term energy management capabilities.
The average execution delay of $T_{task}$ shows a continuous and significant increase from about $110s$ to about $192s$. The MADDPG algorithm exhibits a similar trend, but with a higher average execution delay compared to our algorithm. This is mainly due to the “path dependency” between tasks. The higher task load means that AUVs need to maneuver across the region over longer distances, which inevitably increases the average time spent moving.

\subsection{Covertness Quantification}
 \begin{figure}
    \centering
    	\captionsetup{font={small}} 
    \includegraphics[width=1\linewidth]{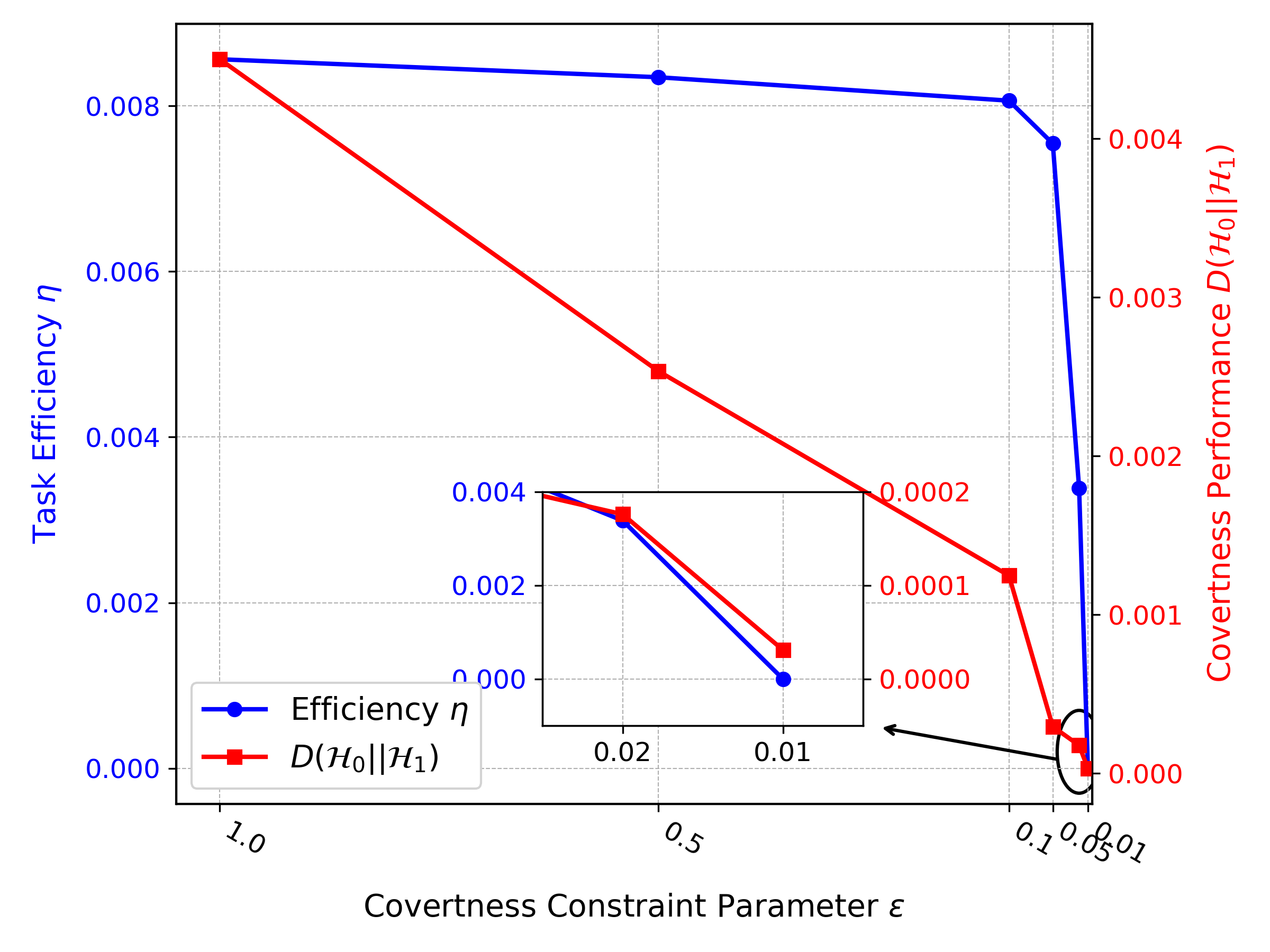}
    \caption{The effect of the covertness constraint parameter on task efficiency and covertness performance.}
    \label{covert}
\end{figure}
To evaluate the performance of our proposed HMAPPO framework under varying levels of covert requirements, we investigate the relationship between system coordination efficiency and communication covertness (measured by average KL divergence) across a set of discrete covertness constraint parameters $\varepsilon \in \{1,0.5,0.1,0.05,0.02,0.01\}$. 
As shown in Fig.~\ref{covert}, when there is no covertness constraint or the constraint is relatively loose, the system prioritizes coordination performance, and task efficiency remains consistently near the performance upper bound of approximately $0.0085$.
Under these conditions, AUVs can leverage higher transmission power to maintain high-quality communication, thereby optimizing overall coordination efficiency. However, this efficient collaboration comes at the expense of covertness, and its high KL divergence indicates that the system's communication behavior is at a higher risk of detection.
As the covertness constraint is gradually tightened (from $0.5$ to $0.05$), the system's behavioral policy undergoes significant changes. HMAPPO is capable of strategically regulating the transmission power of AUVs, enabling the KL divergence to steadily decrease to below $0.001$, thereby effectively ensuring communication covertness.

However, when the $\varepsilon$ enters the stringent range of $0.05$ to $0.01$, a performance cliff phenomenon can be observed. As shown in the magnified inset of the figure, task efficiency $\eta$ experiences a steep drop. As $\varepsilon$ decreases from $0.05$ to $0,02$, $\eta$ falls from $0.0075$ to $0.0034$, which represents a decline of over $54.67\%$. This indicates that the system’s coordination capability suffers fundamental degradation in order to meet extremely strict covertness requirements. Within this range, AUVs are forced to adopt ultra-low transmission power, rendering the acoustic communication links critical for central AUV command and control highly unreliable. As a result, disseminating task commands, real-time reporting of AUV states, and transmitting detection results become exceptionally difficult. These disruptions directly lead to the failure of certain sub-tasks, ultimately causing the overall coordination efficiency to approach zero.

\subsection{Multi-Dimensional Performance Evaluation}
\begin{figure}
    \centering
    	\captionsetup{font={small}} 
    \includegraphics[width=1\linewidth]{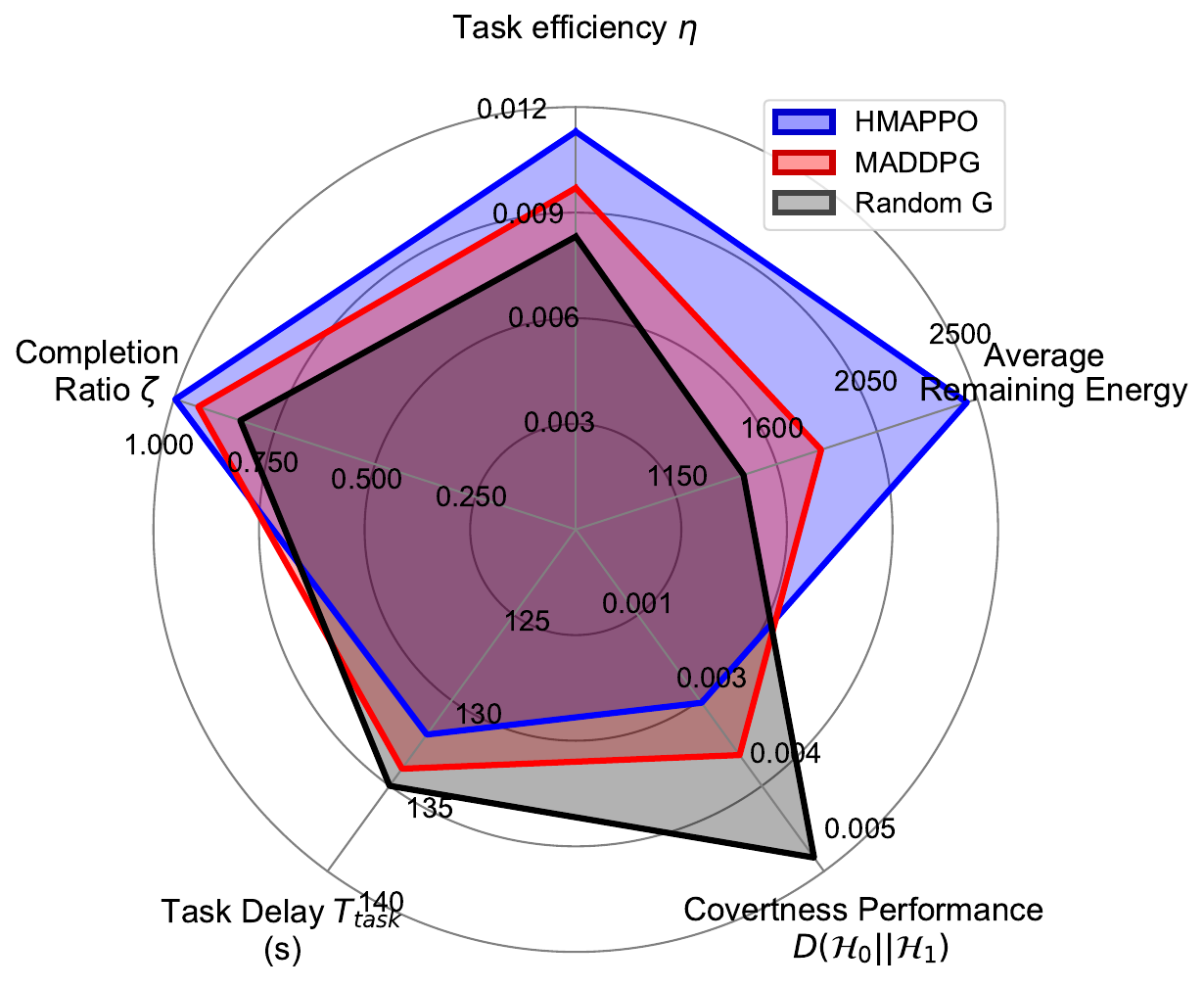}
    \caption{Comprehensive trade-off comparison of different algorithms across five key performance dimensions.}
    \label{radar}
\end{figure}
To comprehensively evaluate the overall effectiveness of our proposed HMAPPO framework from a global perspective, this section provides an intuitive visualization of the trade-offs between five key objectives: task efficiency $\eta$, task completion ratio $\zeta$, average remaining energy, task delay $T_{task}$, and covertness performance $D( {\mathcal{H}_{0}}||{\mathcal{H}_{1}})$. For benefit-oriented metrics ($\eta$, $\zeta$, and average remaining energy), data points farther from the center indicate better performance, whereas for cost-oriented metrics ($T_{task}$ and $D( {\mathcal{H}_{0}}||{\mathcal{H}_{1}})$), points closer to the center represent superior outcomes. 

Fig.~\ref{radar} illustrates that our proposed HMAPPO framework (blue area) demonstrates the most balanced and optimal trade-off strategy, with its performance envelope outperforming all baseline algorithms across every dimension. In terms of task-related performance, HMAPPO achieves near-perfect execution, with both task efficiency and completion ratio reaching the performance boundary. More importantly, this exceptional task performance is achieved with high energy efficiency, indicating that HMAPPO not only discovers optimal cooperation and trajectory strategies but also precisely manages energy consumption at the micro-level.
In terms of execution cost, the HMAPPO also performs exceptionally well. It completes tasks with the lowest delay and achieves the smallest KL divergence, ensuring the best covertness. In comparison, MADDPG (red area) exhibits a relatively balanced strategy. However, its higher energy consumption and longer delay increase the risk of exposure. The random strategy (black area) performs the worst in all aspects, showing very poor trade-off ability.

\section{Conclusions} \label{conclusion}
This paper studied efficient multi-AUV coordination for target detection in complex underwater acoustic environments while maintaining cooperative communication confidentiality. A novel hierarchical multi-AUV PPO framework was presented that decomposes long-horizon tasks into macro-level AUV scheduling and micro-level AUV trajectory control. By leveraging the CTDE paradigm, the proposed framework effectively designs coordination strategies that strike a balance between collaboration detection performance and cooperative communication covertness. The presented extensive simulation results verified that the presented HMAPPO framework converges efficiently, enabling AUVs to collaboratively complete predefined tasks while adhering to covertness constraints. The comparative experiments with baseline algorithms demonstrated that our method achieves superior performance across multiple optimization objectives under varying task loads. In addition, the inherent trade-off between coordination efficiency and communication covertness was analyzed, identifying a critical threshold region where the system can accomplish complex cooperative tasks efficiently while maintaining a sufficiently reliable level of covertness.

\ifCLASSOPTIONcaptionsoff
  \newpage
\fi

\vspace{12pt}

%He is co-inventor of 42 issued US patents,

% biography section
% 
% If you have an EPS/PDF photo (graphicx package needed) extra braces are
% needed around the contents of the optional argument to biography to prevent
% the LaTeX parser from getting confused when it sees the complicated
% \includegraphics command within an optional argument. (You could create
% your own custom macro containing the \includegraphics command to make things
% simpler here.)
%\begin{IEEEbiography}[{\includegraphics[width=1in,height=1.25in,clip,keepaspectratio]{mshell}}]{Michael Shell}
% or if you just want to reserve a space for a photo:

%\begin{IEEEbiography}{Michael Shell}
%Biography text here.
%\end{IEEEbiography}
%
%% if you will not have a photo at all:
%\begin{IEEEbiographynophoto}{John Doe}
%Biography text here.
%\end{IEEEbiographynophoto}
%
%
%\begin{IEEEbiographynophoto}{Jane Doe}
%Biography text here.
%\end{IEEEbiographynophoto}

% You can push biographies down or up by placing
% a \vfill before or after them. The appropriate
% use of \vfill depends on what kind of text is
% on the last page and whether or not the columns
% are being equalized.

%\vfill

% Can be used to pull up biographies so that the bottom of the last one
% is flush with the other column.
%\enlargethispage{-5in}

% that's all folks
\end{spacing}
\end{document}